\documentclass{amsart}
\usepackage{graphicx} 

\usepackage{hyperref}
\usepackage{url}
\usepackage{amsmath}
\usepackage{amsthm}
\usepackage{amssymb}
\usepackage[numbers]{natbib}
\usepackage{xcolor}
\usepackage{cancel}

\theoremstyle{definition}
\newtheorem{definition}{Definition}[section]

\theoremstyle{plain}
\newtheorem{theorem}{Theorem}
\newtheorem{lemma}{Lemma}

\newtheorem{corollary}{Corollary}

\theoremstyle{remark}
\newtheorem{remark}{Remark}
\newtheorem{example}{Example}

\makeatletter
\renewcommand\subsection{\@startsection{subsection}{2}\z@{-.5\linespacing\@plus-.7\linespacing}{.5\linespacing}{\normalfont\scshape}}
\renewcommand\subsubsection{\@startsection{subsubsection}{3}\z@{.5\linespacing\@plus.7\linespacing}{-.5em}{\normalfont\scshape}}
\makeatother

\usepackage{amsmath,amsfonts,bm}
\usepackage{stackengine}

\DeclareMathAlphabet{\mathsfit}{\encodingdefault}{\sfdefault}{m}{sl}
\SetMathAlphabet{\mathsfit}{bold}{\encodingdefault}{\sfdefault}{bx}{n}

\newcommand{\E}{\mathbb{E}}

\newcommand{\R}{\mathbb{R}}
\newcommand{\N}{\mathbb{N}}

\newcommand{\Var}{\mathrm{Var}}

\DeclareMathOperator*{\argmin}{arg\,min}
\DeclareMathOperator*{\im}{im}

\DeclareMathOperator{\dom}{dom}

\newcommand{\norm}[1]{\left\lVert#1\right\rVert}
\newcommand{\magn}[1]{\left\lvert#1\right\rvert}

\title{Thermodynamic structure of the Sinkhorn flow}
\author{Anand Srinivasan}
\address{Department of Applied Mathematics and Theoretical Physics, University of Cambridge}
\email{as3273@cam.ac.uk}

\author{Jean-Jacques Slotine}
\address{Nonlinear Systems Laboratory, Massachusetts Institute of Technology}
\email{jjs@mit.edu}
\date{January 2026}

\begin{document}

\begin{abstract}
Entropy-regularized optimal transport, which has strong links to the Schr\"odinger bridge problem in statistical mechanics, enjoys a variety of applications from trajectory inference to generative modeling.
A major driver of renewed interest in this problem is the recent development of fast matrix-scaling algorithms\textemdash known as iterative proportional fitting or the Sinkhorn algorithm\textemdash for entropic optimal transport, which have favorable complexity over traditional approaches to the unregularized problem.
Here, we take a perspective on this algorithm rooted in the thermodynamic origins of Schr\"odinger's problem and inspired by the modern geometric theory of diffusion: is the Sinkhorn flow (viewed in continuous-time as a mirror descent by recent results) the gradient flow of entropy in a formal Riemannian geometry?
We answer this question affirmatively, finding a nonlocal Wasserstein gradient structure in the dynamics of its free marginal.
This offers a physical interpretation of the Sinkhorn flow as the stochastic dynamics of a particle with law evolving by the nonlocal diffusion of a chemical potential.
Simultaneously, it brings a standard suite of functional inequalities characterizing Markov diffusion processes to bear upon its geometry and convergence.
We prove an entropy-energy (de Bruijn) identity, a Poincar\'e inequality, and a Bakry-\'Emery-type condition under which a logarithmic Sobolev inequality (LSI) holds and implies exponential convergence of the Sinkhorn flow in entropy.
We lastly discuss computational applications such as stopping heuristics and latent-space design criteria leveraging the LSI and, returning to the physical interpretation, the possibility of natural systems whose relaxation to equilibrium inherently solves entropic optimal transport or Schr\"odinger bridge problems.

 \end{abstract}

\maketitle

\section{Introduction}

A central question at the intersection of statistical mechanics and inference is the following large-deviations problem: given empirical observations of a stochastic process $X_\tau$ at some times $\tau=0$ and $\tau=1$, what is the most likely path taken between $X_0$ and $X_1$ supposing that we have a \textit{prior} on the law of $X_\tau$?
An early instance of such a problem is Schr\"odinger's thought experiment \cite{leonard_survey_2013, chen_stochastic_2021}: suppose that a system of $N$ gas particles undergoing diffusion at thermodynamic equilibrium present a particular \textit{empirical distribution}
\begin{align}
    \label{eq:empirical}
    \mu_\tau^N &= \frac1N \sum_{i=1}^N \delta_{X_\tau^i},
\end{align}
where $\delta_x$ is the Dirac measure centered at $x$. 
What is the probability of observing \eqref{eq:empirical}, given the prior that $X_\tau^i$ are i.i.d. Brownian motions?

In modern terms (see \cite{follmer_random_1988} \S 4.1, \cite{leonard_survey_2013} \S 6, and \cite{dembo_large_2010} Ch. 6-7), this question is formalized and answered in terms of a \textit{large deviation principle} (LDP): informally, 
\begin{align}
    \label{eq:ldp}
    \lim_{N\to\infty}-\frac{1}{N}\log\Pr[\mu^N \in A] =\inf_{P \in A}H(P|R)
\end{align}
where the empirical $\mu^N = (\mu_\tau^N)_{\tau=0}^1$, the candidate $P$, and the prior $R$ are regarded as positive measures (\textit{path measures}) in ${\rm M}_+(\Omega)$ on the space of continuous paths $\Omega = C([0, 1], \R^d)$.
(In \eqref{eq:ldp}, $\mu^N$ and $P$ are probability measures in ${\rm P}(\Omega) \subset {\rm M}_+(\Omega)$ on the path space, whereas $R$ may be an unbounded measure, such as the stationary Wiener measure on $\R^d$.)
The \textit{rate function} of the LDP \eqref{eq:ldp} is the relative entropy 
\begin{align}
    \label{eq:H}
    H(P|R) &= \begin{cases}
    \E_P[\log\frac{dP}{dR}] & {\rm if}\ P \ll R\\
    +\infty & {\rm otherwise}.
    \end{cases}
\end{align} 
Schr\"odinger's original problem is one of inference: suppose one observes snapshots
\begin{align}
    \label{eq:snapshots}
    \mu_0^N \in B_\Delta(\mu),\quad \mu_1^N \in B_\Delta(\nu)
\end{align}
($\Delta$-close in a metrization of the narrow topology, recalling that narrow convergence here means $\int_{\R^d} fd\mu_0^N \to \int_{\R^d}fd\mu$ for all bounded continuous $f$).
The law of large numbers (LLN) states that the left marginal $\mu$ is the limit (in the same topology) of the initial condition $\mu_0^N$.
In the case that $R$ is a Markov process with transition kernel $p(\tau, x, \tau', x')$, this further specifies the terminal condition, giving the extent to which right marginal $\nu$ deviates from the LLN (using $\mu, \nu$ as densities):
\begin{align}
    \label{eq:lln_deviation}
    \int_{\R^d} p(0, x, 1, x')\mu(x)dx &\ne \nu(x').
\end{align}
What is the most likely path (a \textit{conditional} LLN) if one were to constrain $\mu_0^N \to \mu_0^* \in B_\Delta(\mu)$ and $\mu_1^N \to \mu_1^* \in B_\Delta(\nu)$ as $\Delta \downarrow 0$?
Stated as in \S 6 \cite{leonard_survey_2013}, this is 
\begin{align}
    \label{eq:schr_question}
    \lim_{\Delta \downarrow 0}\lim_{N\to\infty} \Pr[\mu^N \in A | \mu_1^N \in B_\Delta(\nu)] \in {\rm P}(\Omega),
\end{align}
which is a probability measure defined on measurable subsets $A$ of the path space $\Omega$.
This is answered by the LDP \eqref{eq:ldp} as the path minimizing the rate function subject to the constraints (see 6.4, \cite{leonard_survey_2013}):
\begin{align}
    \label{eq:dynamic_sb}
    \inf_{P_0=\mu, P_1=\nu} H(P|R),
\end{align}
which is also called the \textit{dynamic Schr\"odinger bridge problem}.

This problem of path inference given snapshots (also called \textit{trajectory inference} \cite{lavenant_toward_2024}) can be posed for a wide variety of stochastic processes outside of equilibrium thermodynamics.
This includes inference of the laws of Fokker-Planck tracers in a fluid flow, those of genetic traits in a species given a stochastic evolutionary model \cite{pavon_span_2021}, and those of high-dimensional generative models which seek to bridge noise and data distributions using optimal stochastic interpolants \cite{de_bortoli_diffusion_2021, albergo_stochastic_2023}.
The second \cite{schiebinger_optimal-transport_2019,lavenant_toward_2024}, in particular, has seen a recent surge in interest due to the development of single-cell RNA sequencing (scRNA-seq), through which trajectory inference can reveal high-dimensional temporal processes associated to cell differentiation.
In fact, no apparent restriction on the law of the prior $R$ has been made: only (i) that the LDP \eqref{eq:ldp} holds (rigorously, by Sanov's theorem \S 6.2, \cite{dembo_large_2010})
and (ii) that $H$ \eqref{eq:H} is well-defined for possibly unbounded path measures $R$ as noted earlier (e.g. the Brownian prior, which has the Lebesgue measure as instantaneous marginals).
For the latter, a sufficient condition is that $R$ is $\sigma$-finite \cite{leonard_properties_2013}. 

Under these assumptions, $H$ satisfies the following chain rule (Theorem 2.4, \cite{leonard_properties_2013}, where it is called the \textit{additive property}):
\begin{align}
    \label{eq:H_chain}
    H(P|R) &= H(P_{01}|R_{01}) + \int_{\R^d\times \R^d} H(P^{xy}|R^{xy})P_{01}(dx,dy)
\end{align}
with $P_{01}, R_{01}$ the joint measures (which we call \textit{couplings} as in the optimal transport literature) of the $t=0, 1$ marginals and $P^{xy}$ denotes the law of the \textit{bridge} $P^{xy} = P(\cdot|X_0=x, X_1=y)$.
In this way, a key reduction of Schr\"odinger's problem is obtained: the cost \eqref{eq:H_chain} decomposes into a \textit{static} cost of the coupling $P_{01}$ plus a \textit{dynamic} cost of the bridge $P^{xy}$.
The dynamic cost can always be made to vanish by setting $P^{xy}$ to the law of the prior, in this case the Brownian bridge $R^{xy}$.
Thus, with $R_{01}$ the heat kernel of variance $\varepsilon$ on $\R^d$,
\begin{align}
    R_{01}(x, y) &= p(0, x, 1, y),\quad p(s, x, t, y) = \frac{1}{(2\pi\varepsilon (t-s))^{d/2}}\exp\left(-\frac{\magn{x-y}^2}{2\varepsilon(t-s)}\right),
\end{align}
we have that \eqref{eq:dynamic_sb} is further equivalent to the \textit{static Schr\"odinger bridge problem}
\begin{align}
    \label{eq:static_sb}
    \inf_{\pi \in \Pi(\mu, \nu)}H(\pi|R_{01}) &\equiv \inf_{\pi \in \Pi(\mu, \nu)} \frac{1}{\varepsilon}\E_{\pi}\left[\frac12\magn{x-y}^2\right] + H(\pi),
\end{align}
where $\Pi(\mu, \nu)$ denotes the set of couplings between marginals $\mu, \nu$. 
Since the marginals are fixed as constraints and $H(\pi) = H(\pi|\mu\otimes \nu) + H(\mu) + H(\nu)$, the static problem \eqref{eq:static_sb} is finally equivalent to
\begin{align}
    \label{eq:eot_intro}
    \inf_{\pi \in \Pi(\mu, \nu)} \E_{\pi}[c] + \varepsilon H(\pi|\mu\otimes \nu),
\end{align}
the \textit{entropy-regularized optimal transport (OT$_\varepsilon$) problem} \cite{peyre_computational_2020}.
Schr\"odinger's problem corresponds to Euclidean cost $c(x, y) = \frac12\magn{x-y}^2$; however,  the dynamic problem \eqref{eq:dynamic_sb} corresponding to any Markov prior $R$ with transition kernel $p(0, x, 1, y) \propto e^{-c/\varepsilon}$ reduces in the same way to entropic OT \eqref{eq:eot_intro}.

We refer to \cite{nutz_introduction_nodate} for a survey of entropic OT independent of the origin in statistical mechanics.
(Keeping to the origin in thermodynamics, however, we keep the focus on the continuous rather than discrete setting.)
In its own right, entropic regularization of $\E_{\pi}[c]$ \eqref{eq:static_sb} is in an analytical improvement: $H$ is lower-semicontinuous with respect to narrow convergence and strictly convex on bounded-$H$ sets (\S 1 \cite{nutz_introduction_nodate}).
The set of couplings $\Pi(\mu, \nu)$ is narrowly compact (see Theorem 4.1, \cite{villani_optimal_2009}), convex, and closed in variation when $\Omega$ is a Polish space (e.g. the path space with supremum norm).
If it contains a bounded-$H$ set, existence and uniqueness of solutions to \eqref{eq:static_sb} is guaranteed (Theorem 1.10, \cite{nutz_introduction_nodate}).
Moreover, under standard existence and uniqueness criteria for the unregularized (Monge) problem (Theorem 10.28, \cite{villani_optimal_2009}), the optimal unregularized coupling concentrates on a $d$-dimensional subset, the graph $(x, T(x))$ of an optimal transport map $T$, and is hence singular with respect to the Lebesgue measure on $X\times Y$.
This especially complicates numerical analysis.
On the other hand, with the entropic regularization equivalent to the static Schr\"odinger bridge problem \eqref{eq:static_sb}, standard optimality criteria (Theorem 2.1, \cite{nutz_introduction_nodate}) shows that the optimal entropic coupling $\pi_*^\varepsilon$ satisfies 
\begin{align}
    \label{eq:eot_optimal}
    \frac{d\pi_*^\varepsilon}{dR_{01}}(x, y) &= e^{f(x)+g(y)}\quad R_{01}\text{-a.e.}
\end{align}
for some \textit{Schr\"odinger potentials} $f \in L^1(\mu), g \in L^1(\nu)$, and is hence strictly positive on the support of the reference measure $R_{01}$.

The prospect of identifying solutions, either as a static Schr\"odinger bridge in the continuous case or as a coupling matrix in the discrete case, especially benefits from entropic regularization: the celebrated \textit{Sinkhorn algorithm} (see \cite{cuturi_sinkhorn_2013} and \S6, \cite{nutz_introduction_nodate}) can be derived by iteratively solving for Schr\"odinger potentials $f$ and $g$ in the optimality condition \eqref{eq:eot_optimal}, which corresponds to fixing the left and right marginals at each half-step in \eqref{eq:sinkhorn}.
In the discrete case, this is a simple matrix scaling algorithm which incurs $O(n^2)$ per iteration for $n$ support points (with $n$ scaling e.g. as $m^d$ for $m$ coordinate-wise bins of some underlying $d$-dimensional domain) versus $O(n^3\log n)$ using traditional linear programming methods for the unregularized problem.
(This is the best case possible, with $n^2$ updates needed under any approach for the entries of a suboptimal coupling; what remains to improve is the number of iterations, or convergence rate.)
As such, both the computation of entropic OT solutions in their own right as well as estimation of unregularized OT couplings by passing to the ``low-temperature'' limit $\varepsilon \to 0$ has seen a recent surge of applications in high-dimensional problems, from training losses for generative models \cite{genevay_learning_2018} to the mentioned trajectory inference problem in single-cell biology \cite{schiebinger_optimal-transport_2019}.

Seeking minima of an entropically regularized cost or potential function is ubiquitous in statistical physics: the entropic OT problem \eqref{eq:eot_intro} is equivalent to minimizing the Gibbs free energy \eqref{eq:static_sb} at temperature $\varepsilon$. 
This arises from the sequence of reductions from Schr\"odinger's dynamic problem \eqref{eq:dynamic_sb} which involved paths with Gaussian fluctuations of variance $\varepsilon$ (for which there are noise-free versions of the thought experiment that lead naturally to Wasserstein geodesics, p. 31 \cite{leonard_survey_2013}).
Yet curiously, despite the fact that the Sinkhorn algorithm is derived so naturally from the optimality system \eqref{eq:eot_optimal} (and later \eqref{eq:schr_map}), it is not (in an obvious manner, at least) a Wasserstein gradient descent of the free energy \eqref{eq:static_sb}, as one would expect from the nonequilibrium statistical physics point of view \cite{markowich_trend_nodate}.
This would result in a (marginal-constrained) Fokker-Planck or Langevin dynamics on the joint space $X\times Y$, which was in fact investigated in \cite{conforti_projected_2023} and appears to be distinct.
On the other hand, viewing the Sinkhorn algorithm via combined half-steps (so one marginal is fixed along each full step, with the other free and tending asymptotically to the constraint) shows that it is a mirror gradient descent in the space of probability measures \cite{aubin-frankowski_mirror_2022} with dual variable precisely the Schr\"odinger potential corresponding to the free marginal.
This interpretation is closer in spirit to Schr\"odinger's original question \eqref{eq:schr_question}: we fix by observation one marginal at $\mu$ and ask for the limiting path measure as the empirical distribution converges to the terminal constraint $\nu$. 
Further work \cite{karimi_sinkhorn_2024} showed that, by incorporating the one-sided constraint set into the definition of the mirror map, a continuous-time limit of this mirror descent exists as a coupled pair of partial differential equations, called the Sinkhorn flow.

In this work, we take a perspective on the Sinkhorn flow rooted in the statistical physics viewpoint of minimizing the entropic OT \eqref{eq:eot_intro} or static Schr\"odinger bridge \eqref{eq:static_sb} costs: is there a \textit{thermodynamic structure}, by which we mean a gradient flow of relative entropy in some Riemannian metric, in the relaxation dynamics of the free marginal?
As with the Fokker-Planck equation \cite{jordan_variational_1998}, this variational perspective offers more than a physically plausible gradient flow: it potentially brings Bakry-\'Emery theory and the corresponding suite of functional inequalities characterizing diffusion processes \cite{bakry_analysis_2014} to bear upon the geometry and convergence of the Sinkhorn flow.
Convergence analysis of the Sinkhorn algorithm has a rich history, from early results showing that the discrete Sinkhorn iteration is a contraction mapping in a Hilbert projective metric (see \S8.3 \cite{chen_stochastic_2021} for a review and \cite{eckstein_hilberts_2025} for an analogous result in the continuous setting with unbounded cost) to recent work showing exponential convergence in $L^1$ and Wasserstein under mild log-concavity assumptions on the marginals \cite{conforti_quantitative_2024}.
The thermodynamic perspective we take on Sinkhorn here complements this prior analysis with geometric criteria: as is often the case with diffusion processes, the metric defining the entropic gradient yields a Dirichlet form for whom a Poincar\'e-type inequality yields convergence in entropy.
The key criteria in this framework amount to coercivity properties of conditional expectation operators associated to the coupling $\pi_t$ along the flow.

This paper is organized in three sections.
In the first, we show that the Sinkhorn mirror flow is exponentially contracting in the Fisher-Rao metric
if a certain coercivity property of conditional expectations holds.
In the second, we examine one of the terms arising from this coercivity criterion and show that it is in fact the exact entropy decay rate of the Sinkhorn flow.
We find that this is an instance of the de Bruijn identity relating entropy dissipation and the Fisher information typically governing diffusion processes.
In doing so, we answer the above question affirmatively: the free-marginal dynamics of the Sinkhorn flow is in fact the gradient flow of relative entropy (relative to the imposed terminal constraint \eqref{eq:schr_question}) in the Riemannian metric defined by the inverse of a time-inhomogeneous nonlocal diffusion operator.
This is a type of nonlocal Wasserstein space, studied previously in the context of gradient-flow formulations of jump processes.
In the third, we exploit this structure as in the geometry of diffusion to obtain strictly positive entropy dissipation rates, which is made quantitative via a Poincar\'e inequality (no coercivity of conditional expectations is required, only positive entropy regularization), and show that the convergence in entropy is exponential if and only if a logarithmic Sobolev-type inequality holds.
The latter is another coercivity-type property for conditional expectations, as we show via a sufficient condition for the log-Sobolev inequality in the form of a carr\'e-du-champ estimate.
We conclude these results with some practical applications of the log-Sobolev constant, including geometric design criteria for latent spaces in generative models and stopping heuristics.
Lastly, we discuss the physical viewpoint initiated by this question: are there natural systems which provably solve entropic optimal transport or Schr\"odinger bridge problems?

\section{Background}

\subsection{Entropic optimal transport}

We recall the entropy-regularized optimal transport problem \cite{nutz_introduction_nodate}.
Let $X, Y = \R^d$ and $\mathcal{P}(X), \mathcal{P}(Y)$ the corresponding sets of probability measures.
Let $\mu \in \mathcal{P}(X), \nu \in \mathcal{P}(Y)$ be given marginals and $\Pi(\mu, \nu) \subset \mathcal{P}(X\times Y)$ denote the set of \textit{couplings} between these marginals,  i.e. those joint probability measures $\pi \in \mathcal{P}(X\times Y)$ satisfying 
\begin{align}
    \label{eq:left_marginal}
    \int_{X\times Y} f(x)\pi(dx, dy) &=: \int_X f(x)\pi^X(dx) = \int_X f(x)\mu(dx)\\
    \label{eq:right_marginal}
    \int_{X\times Y}g(y)\pi(dx, dy) &=: \int_Y g(y) \pi^Y(dy) = \int_Yg(y)\nu(dy),
\end{align}
for all bounded measurable $f, g$, where  we have defined the marginalizations $(\cdot)^X, (\cdot)^Y$, used throughout.
Equations \eqref{eq:left_marginal}, \eqref{eq:right_marginal} are also the sense in which we denote equality of measures, e.g. $\pi^X = \mu$.
Given a distance cost function $c(x, y) : X\times Y \to \R_+$, the entropic optimal transport problem (denoted ${\rm OT}_\varepsilon$) is, as stated in \eqref{eq:eot_intro},
\begin{align}
    \label{eq:eot}
    {\rm OT}_\varepsilon(\mu, \nu) &= \min_{\pi \in \Pi(\mu, \nu)} \E_\pi[c] + \varepsilon H(\pi | \mu \otimes \nu),
\end{align}
where $H$ is the relative entropy, defined for probability measures as in \eqref{eq:H}.
The existence of the minimizer (denoted throughout as $\pi_*$) in \eqref{eq:eot} is due to the lower semicontinuity of $H$ and narrow compactness of $\Pi(\mu, \nu)$ as noted earlier (see e.g. Theorem 1.10, \cite{nutz_introduction_nodate}).
The Fenchel-Rockafellar dual form of ${\rm OT}_\varepsilon$ \eqref{eq:eot} (4.4, \cite{peyre_computational_2020}) is
\begin{align}
    \label{eq:eot_dual}
    {\rm OT}_\varepsilon(\mu, \nu) &= \sup_{f\in L^1(\mu),\ g\in L^1(\nu)}\E_\mu[f] + \E_\nu[g] - \varepsilon\E_{(\mu\otimes \nu)}\left[\exp\left(\frac{-c + (f \oplus g)}{\varepsilon}\right)\right]
\end{align}
where $(f \oplus g)(x, y) = f(x) + g(y)$ and $f \in L^1(\mu), g \in L^1(\nu)$ are referred to as the \textit{Schr\"odinger potentials} of the solution.
The dual optimality conditions are given by the \textit{Schr\"odinger system}
\begin{align} 
    \label{eq:schr_map}
    \mathcal{S}: (\mu, \nu) &\mapsto (f, g),\quad 
    \begin{cases}
        f(x) = -\varepsilon\log\int_Y\exp\left(\frac{g(y)-c(x, y)}{\varepsilon}\right)\nu(dy)\\
        g(y) = -\varepsilon\log\int_X \exp\left(\frac{f(x) - c(x, y)}{\varepsilon}\right)\mu(dx).
    \end{cases}
\end{align}

\subsection{The Sinkhorn algorithm for ${\rm OT}_\varepsilon$ as a mirror descent}

Throughout, we use $t$ to denote \textit{algorithmic time}, distinct from the \textit{bridge time} $\tau$ \eqref{eq:empirical}.
Denoting the sets of joint probability measures satisfying the one-sided marginal constraints by
\begin{align}
    \Pi(\mu, \cdot) = \{\pi \in \mathcal{P}(X\times Y)\ |\ \pi^X = \mu\},\ \Pi(\cdot, \nu) = \{\pi \in \mathcal{P}(X\times Y)\ |\ \pi^Y=\nu\},
\end{align}
the \textit{Sinkhorn algorithm} \cite{cuturi_sinkhorn_2013, peyre_computational_2020} is an \textit{iterative proportional fitting procedure} (IPFP)
\begin{equation}
\begin{split}
    d\pi_0 &\propto \exp(-c/\varepsilon)d(\mu\otimes \nu)\\
    \pi_{t+\frac12} &= \argmin_{\pi \in \Pi(\cdot, \nu)}H(\pi | \pi_{t})\\
    \pi_{t+1} &= \argmin_{\pi \in \Pi(\mu, \cdot)} H(\pi | \pi_{t+\frac12}),
\end{split}
\label{eq:sinkhorn}
\end{equation}
whose stationary state is the minimizer $\pi_*$ of ${\rm OT}_\varepsilon$ \eqref{eq:eot}.
Standard optimality conditions (\S 6, \cite{nutz_introduction_nodate}) show that the iterations \eqref{eq:sinkhorn} are ``column'' and ``row'' normalizations 
\begin{align}
    d\pi_{t+\frac12}(x, y) &= d\pi_t(x, y)\frac{d\nu}{d\pi_t^Y}(y)\\
    d\pi_{t+1}(x, y) &= d\pi_{t+\frac12}(x, y)\frac{d\mu}{d\pi_{t+\frac12}^X}(x).
\end{align}
This primal iteration corresponds precisely to iteratively applying the dual optimality conditions of the Schr\"odinger system \eqref{eq:schr_map} (Algorithm 6.1, \cite{nutz_introduction_nodate}).
This corresponds to an alternating Bregman projection (with relative entropy $H$ the Bregman divergence) whose basic properties (\S 6.1, \cite{nutz_introduction_nodate}) are 
\begin{align}
H(\pi_*|\pi_t) \le H(\pi_*|\pi_{t-\frac12}),\quad H(\pi_t^X|\mu)+H(\pi_t^Y|\nu) \le H(\pi_t|\pi_{t-\frac12}).
\end{align}

When $\pi_t$ is absolutely continuous with respect to the Lebesgue measure, it was shown in \cite{aubin-frankowski_mirror_2022} that the Sinkhorn algorithm \eqref{eq:sinkhorn} can be written as the constrained mirror descent (MD, also called Bregman gradient descent) \cite{beck_mirror_2003}
\begin{align}
    \label{eq:md_sinkhorn}
    \pi_{t + 1} &= \argmin_{\pi \in \Pi(\mu, \cdot)}DF(\pi)[\pi-\pi_{t}] + D_\varphi(\pi|\pi_{t}),
\end{align}
for $t \in \N$, where $F$ and $\varphi$ are the \textit{objective} and \textit{mirror map} (Bregman potential)
\begin{align}
    \label{eq:objective}
    F(\pi) &= H(\pi^Y|\nu),\quad DF(\pi)[\eta] \equiv \langle\log\frac{d\pi^Y}{d\nu}, \eta^Y\rangle,\quad \varphi(\pi) = H(\pi | \pi_0)
\end{align}
with $D_\varphi$ the Bregman divergence of $\varphi$.
($F$ and $\varphi$ are only directionally differentiable in the space of measures \cite{aubin-frankowski_mirror_2022}, since the interior of the entropy functional domain is empty due to the nonnegativity constraint.
Thus, what is meant throughout by the Fr\'echet derivative-style notation $DF(\pi)[\eta]$ is the directional derivative
\begin{align}
    \label{eq:directional_derivative}
    DF(\pi)[\eta] &:= \lim_{s\to 0^+}\frac{F(\pi + s\eta) - F(\pi)}{s},
\end{align}
which always exists for convex $F$, though can be infinite.
Note that necessarily, if $\pi + s\eta \in \dom(F)$ then $\eta \in L^1$.
Similarly, we shall define iterated differentials
\begin{align}
    \label{eq:second_directional_derivative}
    D^2F(\pi)[\eta_1, \eta_2] &:= \lim_{s\to0^+} \frac{DF(\pi + s\eta_2)[\eta_1] - DF(\pi)[\eta_1]}{s},
\end{align}
which are also one-sided due to nonnegativity.)
For step size $\gamma > 0$, in \cite{karimi_sinkhorn_2024} the $\gamma$-Sinkhorn MD iteration is defined via $\varphi \mapsto \varphi / \gamma$ as
\begin{align}
    \label{eq:gamma_sinkhorn_md}
    \pi_{t+\gamma} &= \argmin_{\pi \in \Pi(\mu, \cdot)}DF(\pi)[\pi-\pi_t] + \frac{D_\varphi(\pi|\pi_t)}{\gamma}.
\end{align}
Its $\gamma \to 0$ limit was shown to be a continuous-time mirror descent in $(L^1, {\rm P})$, where
\begin{align}
    \label{eq:sinkhorn_dual}
    \frac{\partial h}{\partial t} &= -DF(\pi),\quad h = D\varphi(\pi)
\end{align}
is the flow in the dual space $L^1(X\times Y)$, where the dual variable $h$ is (up to addition of $X$-measurable functions) the right Schr\"odinger potential $g$ in \eqref{eq:eot_dual}.
As seen from \eqref{eq:objective}, $h$ evolves only in the $y$-coordinate.
We generally drop the subscripts $h_t, \pi_t$ in favor of $h(t, x, y), \pi(t, x, y)$, etc. in the continuous-time setting, occasionally using $h_t, \pi_t$, etc. to discuss as functions of spatial variables only.
The flow in the primal space $\Pi(\mu, \cdot) \subset {\rm P}(X\times Y)$ is
\begin{align}
    \label{eq:sinkhorn_primal}
    \pi(t, x, y) &= \left[D\varphi^*(h)\right](t, x, y)=  \frac{\pi(0,x,y)e^{h(t,x,y)}}{\int_{Y}\pi(0,x,y')e^{h(t,x,y')}dy'}\mu(x).
\end{align}
Note that, as in the discrete-time mirror descent \eqref{eq:gamma_sinkhorn_md}, we have $\pi^X(t, \cdot) = \mu$ by construction.
In \eqref{eq:sinkhorn_primal}, $\varphi^*$ is a shorthand for the Fenchel conjugate of the indicator-constrained Bregman potential $\varphi + \iota_{\Pi(\mu, \cdot)}$, 
\begin{align}
    (\varphi+ \iota_{\Pi(\mu, \cdot)})^*(h) &= \sup_{\pi' \in \Pi(\mu, \cdot)} \langle\pi', h\rangle - \varphi(\pi'),\quad \iota_{\Pi(\mu, \cdot)}(\pi) = \begin{cases}
        0 & \pi^X = \mu\\
        +\infty & \text{otherwise,}
    \end{cases}
\end{align}
whose first variation is \eqref{eq:sinkhorn_primal} as shown by Lemma 3, \cite{karimi_sinkhorn_2024}.

\subsection{Notation}
In addition to the notation already defined, we use the following abbreviations for $L^2$ inner products:
\begin{align}
    \label{eq:inner_products}
    \langle f, g\rangle_{\pi} := \langle f, g\rangle_{L^2(\pi)} = \int_{X\times Y} f(x, y)g(x, y)\pi(dx, dy),\quad \langle \cdot, \cdot\rangle := \langle \cdot, \cdot\rangle_{L^2}
\end{align}
with the domains of integration implied by the measure $\pi$.
Moreover, while $f, g, \pi$ may also depend on time (e.g. $\pi(t, x, y)$), inner products \eqref{eq:inner_products} denote integration \textbf{only in space}.
We also denote the subspace of mean-zero functions by 
\begin{align}
    \label{eq:mean_zero}
    L_0^2(\pi) &:= \{f \in L^2(\pi)\ |\ \E_\pi[f] = 0\},
\end{align}
with $L_0^2(\R^d)$ or just $L_0^2$ denoting the mass-zero subspace.
Finally, we shall denote the disintegrations (conditional measures) by
\begin{align}
    \pi(dx, dy) = \pi(dx|y)\pi^Y(dy) = \pi(dy|x)\pi^X(dx),
\end{align}
with equality of measures in the sense of \eqref{eq:left_marginal}.

\subsection{Definitions}

Let us recall for completeness some standard definitions and their properties.

\begin{definition}[Conditional expectations as projection operators]
\label{def:cond_exp}
Let $(X\times Y, \pi)$ be the probability space.
Define
\begin{align}
    \label{eq:P_pi}
    (P_\pi f)(y) &:= \E_\pi[f | Y=y] = \int_{X}f(x, y)\pi(dx|y) \in L^2(\pi^Y),
\end{align}
which is an orthogonal projection (after canonically embedding $L^2(\pi^Y)$ back in $L^2(\pi)$ via $h(x, y) := h(y)$, which we take for granted throughout), since for all $f, g \in L^2(\pi)$, $P_\pi$ is:
\begin{enumerate}
    \item a projection, by the tower property
    \begin{align}
        \label{eq:p_idemp}
        P_\pi P_\pi f &= \E_\pi[\E_\pi[f|Y]|Y] = \E_\pi[f|Y] = P_\pi f,
    \end{align}
    \item self-adjoint, by
    \begin{align}
        \label{eq:p_adj}
        \langle P_\pi f, g\rangle_{L^2(\pi)} &= \iint_{X\times Y}g(x, y)\pi(dx|y)\int_Xf(x', y)\pi(dx', dy) = \langle f, P_\pi g\rangle_{L^2(\pi)}.
    \end{align} 
    \item bounded and a contraction, by Jensen's inequality
    \begin{align}
        \label{eq:P_pi_bounded}
        \norm{P_\pi f}_{L^2(\pi)}^2 &= \E_\pi[(\E_\pi[f|Y])^2] \le \E_\pi[\E_\pi[f^2|Y]] = \norm{f}_{L^2(\pi)}^2.
    \end{align}
    \item the orthogonal projection onto the closed subspace
    \begin{align}
        \im P_\pi &= \{g \in L^2(\pi)\ |\ \exists h \in L^2(\pi^Y)\ {\rm s.t.}\ g(x, y) = h(y)\ {\rm for}\ \pi-{\rm a.e.} (x, y)\}.
    \end{align}
\end{enumerate}
Similarly, define the projection $Q_\pi : L^2(\pi) \to L^2(\pi^X)$
\begin{align}
\label{eq:Q_pi}
(Q_\pi f)(x) &:= \E_{\pi}[f|X=x] = \int_Yf(x, y)\pi(dy|x)
\end{align}
which also has properties (1), (2), and (3) as above with (4) being
\begin{align}
    \im Q_\pi &= \{g \in L^2(\pi)\ |\ \exists f \in L^2(\pi^X)\ {\rm s.t.}\ g(x, y) = f(x)\ {\rm for}\ \pi-{\rm a.e.} (x, y)\}.
\end{align}
\end{definition}

\begin{definition}[Numerical range]
\label{def:num_range}
Let $A \in \mathcal{B}(H)$ be a bounded linear operator on a Hilbert space $H$. 
Its \textit{numerical range} $W(A)$ is the subset of the complex plane 
\begin{align}
    W_H(A) &= \{\frac{\langle v, Av\rangle_H}{\langle v, v\rangle_H}\ |\ v \in H,\ v \ne 0\},
\end{align}
which is equivalently the map of the unit sphere $\norm{v}_H=1$ under $v \mapsto \langle v, Av\rangle_H$.
\end{definition}

\begin{definition}[Coercivity]
We call an operator $A$ as in \ref{def:num_range} $\lambda$-\textit{coercive} in a real Hilbert space $H$ if $\inf W_H(A) = \lambda > 0$.
\end{definition}

\subsection{Assumptions}
\label{sec:assumptions}
For simplicity, we shall assume $\mu, \nu$ are absolutely continuous with respect to the Lebesgue measure in the following and use $\mu, \nu, \pi$ to interchangeably represent measures and densities, e.g. $d\mu(x) = \mu(x)dx$. 
This precludes $\mu$ or $\nu$ being empirical distributions, but we believe many, if not all, of the  arguments presented here can be adapted to the general case.
We also assume that as densities, $\mu, \nu > 0$ Lebesgue-a.e., so that $\pi_0 > 0$ and $\pi_t > 0$ L-a.e. from definitions \eqref{eq:sinkhorn} and \eqref{eq:sinkhorn_primal}. 
Then, the marginals will be expressed as the exponential densities
\begin{align}
    \label{eq:densities}
    d\mu(x) &\propto \exp(-U(x))dx,\quad d\nu(y) \propto \exp(-V(y))dy.
\end{align}

\section{Results}

\subsection{Contraction in weighted $L^2$ space}

From the dual flow \eqref{eq:sinkhorn_dual}, one might expect that the entropy $H(\pi^Y|\nu)$ of the right marginal should decay in time.
Indeed, \cite{karimi_sinkhorn_2024} shows a ``second law'' $dH(\pi_t^Y|\nu)/dt \le 0$ as well as the sublinear convergence rate $H(\pi_t^Y|\nu) = O(t^{-1})$, consistent with known results for the discrete-time Sinkhorn algorithm (\S 6.2, \cite{nutz_introduction_nodate}).
(Here, $H$ is the negative of Shannon's entropy, so this corresponds physically to entropy \textit{production}. 
We also use ``entropy'' and the more precise term ``relative entropy'' interchangeably, with the latter being a change of reference measure.)
On the other hand, from the semigroup theory of diffusions \cite{bakry_analysis_2014}, it is often the case that strongly contractive Markov processes (typically, on weighted $L^2$ space) correspond to exponential convergence in entropy, with the rates of contraction and entropy decay related up to constants. 
Here, we similarly show that entropy decay rates of the Sinkhorn flow are tightly linked to contraction estimates in weighted $L^2$ spaces, finding that\textemdash despite not being a typical diffusion process\textemdash it satisfies the same entropy-energy identity as diffusion.

To establish the setting for geodesic contraction as in the finite-dimensional case \cite{lohmiller1998contraction}, we first give a consistent specification of tangent and cotangent spaces along the flow \eqref{eq:sinkhorn_dual}, \eqref{eq:sinkhorn_primal} which are compatible with the constraint set $\Pi(\mu, \cdot)$ \eqref{eq:md_sinkhorn}.

\begin{definition}[Tangent and cotangent spaces along Sinkhorn mirror flow]
\label{def:tangent_cotangent}
We define the tangent space to the constraint set $\Pi(\mu, \cdot) \subset {\rm P}(X\times Y)$ as 
\begin{align}
    \label{eq:tangent}
    T_\pi\Pi(\mu, \cdot) &:= \{\delta \pi \in L^2(1/\pi)\ |\ \delta\pi^X = 0\ \text{a.e.}\},\quad \pi \in \Pi(\mu, \cdot),
\end{align}
where $\delta \pi^X$ denotes $\int_Y\delta \pi(x, y)dy$.
We also define the cotangent space as
\begin{align}
    \label{eq:cotangent}
    T_\pi^*\Pi(\mu, \cdot) &:= \frac1\pi T_\pi \Pi(\mu, \cdot) = \{\delta h \in L^2(\pi)\ |\ \pi \delta h \in T_\pi\Pi(\mu, \cdot)\},
\end{align}
both of which are well-defined along the Sinkhorn flow \eqref{eq:sinkhorn_primal} since $\pi > 0$ a.e.
The tangent and cotangent spaces are related by the \textit{Fisher-Rao metric} 
\begin{align}
    \label{eq:fisher_rao}
    \langle a, b\rangle_{{\rm FR}(\pi)} &= \iint_{X\times Y} \frac{ab}{\pi}dxdy,\quad a, b\in T_\pi\Pi(\mu, \cdot),
\end{align}
which is guaranteed finite by definition \eqref{eq:tangent}.
Lastly, we note that the cotangent space is equivalently $T_\pi^*\Pi(\mu, \cdot) = \ker Q_\pi$, viewing $Q_\pi$ as an orthogonal projection of $L^2(\pi)$ as in Definition \ref{def:cond_exp}.
\end{definition}

\begin{remark}
We do not employ the tangent space $T_\pi\Pi(\mu, \cdot)$ as a set of feasible directions along which the second differential \eqref{eq:second_directional_derivative} $D^2H(\pi)$ of entropy is assumed to exist.
As mentioned earlier, $\dom(H)$ has empty interior due to the restriction to strictly positive densities.
Therefore, additional constraints are required; for example, a sufficient (but possibly too restrictive) definition is
\begin{align}
    \label{eq:relative_bounded}
    \Tilde{T}_\pi \Pi(\mu, \cdot) &= \{\delta \pi \in T_\pi\Pi(\mu, \cdot)\ |\ \frac{\delta \pi}{\pi} \in L^\infty\},\quad \Tilde{T}_\pi^* \Pi(\mu, \cdot) = \ker Q_\pi \cap L^\infty.
\end{align}
The relative boundedness ensures that $H$ possesses directional derivatives, since there exists $s>0$ such that $\pi + s\delta \pi > 0$ a.e.
This also satisfies the necessary criterion mentioned in \eqref{eq:directional_derivative}, since $\norm{\delta \pi}_{L^1} \le \norm{\delta \pi / \pi}_{L^\infty}$, and moreover subsumes the weaker integrability in the original definition of the tangent space \eqref{eq:tangent}, since $\norm{\delta \pi}_{L^2(1/\pi)} \le \norm{\delta \pi / \pi}_{L^\infty}$.

Under a suitable constraint on tangent directions such as \eqref{eq:relative_bounded}, one recovers the well-known fact from information geometry \cite{amari_information_2016} that the Fisher-Rao metric \eqref{eq:fisher_rao} is the second variation of entropy:
\begin{align}
    \label{eq:d2h}
    D^2H(\pi)[\delta \pi_1, \delta \pi_2] & = \langle\delta\pi_1, \delta \pi_2\rangle_{{\rm FR}(\pi)},\quad \delta \pi_1, \delta \pi_2 \in \Tilde{T}_\pi\Pi(\mu, \cdot).
\end{align}
By an identical argument, $D^2F(\pi)$ is well-defined for directions $\delta \pi \in \Tilde{T}_\pi\Pi(\mu, \cdot)$, and 
\begin{align}
    \label{eq:d2f}
    D^2F(\pi)[\delta \pi, \delta \pi] &= \int_Y \frac{\delta \pi^Y(y)^2}{\pi^Y(y)}dy = \iint_{X\times Y} \left(P_\pi \frac{\delta \pi}{\pi}\right)^2\pi dxdy \le \norm{\frac{\delta \pi}{\pi}}_{L^2(\pi)}^2 < \infty
\end{align}
exists finite, since $P_\pi$ is an orthogonal projection of $L^2(\pi)$ \eqref{eq:P_pi_bounded} and $\delta \pi \in L^2(1/\pi)$ by definition \eqref{eq:tangent}.
Similarly, with $h = D\varphi(\pi)$ by \eqref{eq:sinkhorn_dual}, we recover that the tangent and cotangent spaces are related by the Hessian of the mirror map \eqref{eq:objective}:
\begin{align}
\delta h = D^2\varphi(\pi)[\delta \pi, *] = \frac{\delta \pi}{\pi} \in \Tilde{T}^*_\pi\Pi(\mu, \cdot),
\end{align}
(using $*$ to indicate an argument which is an element of $\Tilde{T}_\pi\Pi(\mu, \cdot)$, i.e. the duality pairing, distinct from $\cdot$ in the latter indicating an unconstrained right marginal.)
\end{remark}

In the following, rather than constructing a space of feasible directions \textit{a priori} for differentials of entropy as in \eqref{eq:relative_bounded}, we shall instead differentiate the flow along curves in the constraint set $\Pi(\mu, \cdot)$, using the tangent and cotangent spaces \eqref{eq:tangent}, \eqref{eq:cotangent} as domains for relevant operators.
This is closer in spirit to the decay of length functionals over curves pushed forward by gradient flows in the Wasserstein space \cite{ambrosio_gradient_2008}.
With the Fisher-Rao metric naturally relating the tangent and cotangent spaces \eqref{eq:tangent}, \eqref{eq:cotangent} along the Sinkhorn mirror flow, we now consider contraction in the induced geodesic distance.

\begin{definition}[Hellinger-Fisher-Rao and restricted HFR distances]
\label{def:hfr}
We recall that the Fisher-Rao metric $\langle \cdot, \cdot\rangle_{\rm FR}$ \eqref{eq:fisher_rao} metrizes ${\rm P}(X\times Y)$ as a geodesic space via the \textit{Fisher-Rao} or \textit{Hellinger-Fisher-Rao} (HFR) distance (see \cite{mielke_notes_2025} for the nonparametric, infinite-dimensional case):
\begin{align}
    \label{eq:hfr}
    d_{{\rm HFR}(E)}^2(\pi, \pi') &= \inf_{\substack{\gamma\in C^1([0,1];E)\\ \gamma(0)=\pi,\ \gamma(1)=\pi'}}\int_0^1\langle \gamma'(s), \gamma'(s)\rangle_{{\rm FR}(\gamma(s))}ds
\end{align}
with $\gamma'(s) \in T_{\gamma(s)} E$ where $E \subseteq {\rm P}(X\times Y)$ is some admissible set.
By $\gamma \in C^1([0, 1], E)$, we mean the square-root embedding into the $L^2$ sphere
\begin{align}
    \langle \gamma'(s), \gamma'(s)\rangle_{{\rm FR}(\gamma)} &= 4\langle \frac{d}{ds} \sqrt{\gamma}, \frac{d}{ds}\sqrt{\gamma}\rangle_{L^2},
\end{align}
i.e. $\sqrt{\gamma} \in C^1([0, 1], L^2)$ hence $\gamma$ is differentiable in $L^1$.
In the following, we take $E \subseteq \Pi(\mu, \cdot)$, the left-marginal-constrained subset.
\end{definition}

We obtain the following criterion for contraction in the Fisher-Rao metric.

\begin{theorem}[Contraction of Sinkhorn flow in Hellinger-Fisher-Rao]
\label{thm:cont_fr}
Let $E \subseteq \Pi(\mu, \cdot) \subset L^1$ be any set such that (i) $E$ is $C^1$ path-connected in $L^1$, i.e. for $\pi, \pi' \in E$ there exists $\gamma \in C^1([0, 1], E)$ with $\gamma(0) = \pi, \gamma(1) = \pi'$, and (ii) the pushforward $\Gamma(s, t)$ of any such path $\gamma(s)$ under the Sinkhorn flow \eqref{eq:sinkhorn_dual}, \eqref{eq:sinkhorn_primal} exists, remains within $E$, and retains this regularity as $\Gamma \in C^1([0, 1]\times [0, \infty), E)$ with commutativity of $\partial/\partial t, \partial/\partial s$.

Then, if for some $\lambda \in \R$, all $\pi \in E$, and cotangent vectors $\xi \in T_\pi^*E$ defined as in \ref{def:tangent_cotangent}, the coercivity
\begin{align}
    \label{eq:contraction_hfr}
    \langle \xi, \left[2P_{\pi} + (I - Q_{\pi})\log\frac{d\pi^Y}{d\nu}\right]\xi\rangle_{\pi} &\ge \lambda \langle \xi, \xi\rangle_{\pi}
\end{align}
holds, then for any pair of particular solutions $\pi_t, \pi_t' \in E$, 
\begin{align}
    \label{eq:hfr_statement}
    d^2_{{\rm HFR}(E)}(\pi_t, \pi_t') &\le e^{-\lambda t}d^2_{{\rm HFR}(E)}(\pi_0, \pi_0'),
\end{align}
i.e. the Fisher-Rao distance is contracting or maximally expanding with rate $\magn{\lambda}$.
\end{theorem}
\begin{proof}
First, let us note that the statement \eqref{eq:contraction_hfr} is well-defined since $\xi \in L^2(\pi)$ by Definition \ref{def:tangent_cotangent}.
Next, we note a useful identity along the mirror flow \eqref{eq:sinkhorn_dual}, \eqref{eq:sinkhorn_primal}:
\begin{align}
    \label{eq:dlogp_1}
    \frac{\partial}{\partial t} \log \pi &= \frac{\partial h}{\partial t} - \frac{\partial}{\partial t}\log Z,\quad Z(t, x) = \int_Y\pi_0(x, y')e^{h(t, x, y')}dy'.
\end{align}
Since the left marginal $\pi^X(t, \cdot) = \mu$ is fixed along the flow, we have
\begin{align}
    \label{eq:diff_marg}
    0 &= \frac{\partial \pi^X}{\partial t} = \left(\frac{\partial \pi}{\partial t}\right)^X = \left(\pi\frac{\partial}{\partial t}\log \pi\right)^X = \mu Q_\pi \left(\frac{\partial}{\partial t}\log \pi\right)
\end{align}
with interchange of differentiation and marginalization justified by hypothesis $\pi(t, \cdot, \cdot) \in C^1([0, \infty), L^1)$ and the fact that the marginalization operators $(\cdot)^X, (\cdot)^Y$ are continuous on $L^1$.
In the last equality, $Q_\pi$ is the conditional expectation defined in \ref{def:cond_exp}.
Since $Z$ is only a function of $(t, x)$, we have from \eqref{eq:dlogp_1} that
\begin{align}
    \frac{\partial}{\partial t}\log Z &= Q_\pi \frac{\partial h}{\partial t}
\end{align}
(which avoids needing to differentiate the log-partition directly) and hence
\begin{align}
    \label{eq:dlogpi_dt}
    \frac{\partial}{\partial t}\log \pi &= (I - Q_\pi) \frac{\partial h}{\partial t} := Q_\pi^\perp \frac{\partial h}{\partial t},
\end{align}
where we have defined $Q_\pi^\perp := (I - Q_\pi)$.

Now, let $\pi(t, \cdot, \cdot)$ and $\hat{\pi}(t, \cdot, \cdot)$ be any pair of particular solutions in $E$.
Let $\gamma(s) \in C^1([0, 1], E)$ be any path between the initial conditions, and $\Gamma(s, t)$ be the pushforward of this map under the flow \eqref{eq:sinkhorn_dual}, \eqref{eq:sinkhorn_primal}, i.e. 
\begin{align}
    \Gamma(0, t) &= \pi_t,\ \Gamma(1, t) = \hat{\pi}_t,\ \Gamma(s, 0) = \gamma(s),
\end{align}
which remains within $E$ by forward invariance.
We may assume the Fisher-Rao metric $\langle \gamma'(s), \gamma'(s)\rangle_{{\rm FR}(\gamma(s))}$ is finite along the initial path $\gamma$ as otherwise the upper bound \eqref{eq:hfr_statement} is trivial.
Correspondingly, denote the pushforward of the dual variable \eqref{eq:sinkhorn_dual} under the flow by $\Xi(s, t)$.
We can apply an identical argument as in \eqref{eq:diff_marg} in the $s$ variable since, for each fixed $t_*$, we have $\Gamma(s, t_*, \cdot, \cdot) \in C^1([0, 1], L^1)$.
It also satisfies the bound
\begin{align}
    \norm{\frac{\partial \Gamma}{\partial s}}_{L^1} &\le 2\norm{\sqrt{\Gamma}}_{L^2}\norm{\partial_s\sqrt{\Gamma}}_{L^2} = 2\norm{\frac{\partial}{\partial s}\sqrt{\Gamma}}_{L^2} = \sqrt{\langle \frac{\partial \Gamma}{\partial s}, \frac{\partial \Gamma}{\partial s}\rangle}_{{\rm FR}(\Gamma)}.
\end{align}
Thus, we have the pair of identities
\begin{align}
    \label{eq:log_evolutions}
    \frac{\partial}{\partial t}\log \Gamma &= Q_\Gamma^\perp\frac{\partial \Xi}{\partial t},\quad \frac{\partial}{\partial s}\log \Gamma = Q_\Gamma^\perp\frac{\partial \Xi}{\partial s}.
\end{align}
Next, from the dual flow \eqref{eq:sinkhorn_dual}, we have by commutativity of mixed partials 
\begin{align}
    \label{eq:dh_dynamics}
    \frac{\partial}{\partial t} \frac{\partial \Xi}{\partial s} &= -\frac{\partial}{\partial s}\log\frac{\Gamma^Y}{\nu} = -\frac{1}{\Gamma^Y}\frac{\partial \Gamma^Y}{\partial s} = -P_\Gamma \left(\frac{1}{\Gamma}\frac{\partial\Gamma}{\partial s}\right) = -P_\Gamma Q_\Gamma^\perp\frac{\partial \Xi}{\partial s}.
\end{align}
with interchange of differentiation and marginalization in the third equality again justified as before.
Finally, note from \eqref{eq:log_evolutions} that, for some function $g$ such that $g\Gamma$ is $t$-differentiable in $L^1$,
\begin{align}
    \label{eq:Q_dot}
    \frac{\partial}{\partial t}(Q_\Gamma g) &= \frac{1}{\mu}\int_Y(\frac{\partial g}{\partial t} + g\frac{\partial}{\partial t}\log \Gamma)\Gamma dy = Q_\Gamma(\frac{\partial g}{\partial t} + gQ_\Gamma^\perp\frac{\partial \Xi}{\partial t}).
\end{align}
With these ingredients in hand, the Fisher-Rao metric evolves as
\begin{align}
    \frac{d}{dt}\langle \frac{\partial \Gamma}{\partial s}, \frac{\partial \Gamma}{\partial s}\rangle_{{\rm FR}(\Gamma)} &= \frac{d}{dt} \langle Q_\Gamma^\perp \frac{\partial \Xi}{\partial s}, Q_\Gamma^\perp \frac{\partial \Xi}{\partial s}\rangle_\Gamma\\
    &= -2\langle Q_\Gamma^\perp\frac{\partial \Xi}{\partial s}, Q_\Gamma^\perp P_\Gamma Q_\Gamma^\perp\frac{\partial \Xi}{\partial s} + Q_\Gamma(\frac{\partial \Xi}{\partial s}Q_\Gamma^\perp\frac{\partial \Xi}{\partial t})\rangle_\Gamma \\
    &+ \langle Q_\Gamma^\perp \frac{\partial \Xi}{\partial s}, \frac{\partial \log \Gamma}{\partial t}Q_\Gamma^\perp\frac{\partial \Xi}{\partial s}\rangle_\Gamma.
\end{align}
Since $Q_\Gamma, Q_\Gamma^\perp$ are $L^2(\Gamma)$ projections, using the identities \eqref{eq:log_evolutions} and substituting the dual flow \eqref{eq:sinkhorn_dual}, this finally simplifies to
\begin{align}
    \label{eq:hfr_penultimate}
   &= -\langle \frac{1}{\Gamma}\frac{\partial \Gamma}{\partial s}, \left[2P_\Gamma + Q_\Gamma^\perp\log\frac{\Gamma^Y}{\nu}\right]\frac{1}{\Gamma}\frac{\partial \Gamma}{\partial s}\rangle_\Gamma.
\end{align}

Note that since $\Gamma \in E \subseteq \Pi(\mu, \cdot)$, we have $\frac{1}{\Gamma}\frac{\partial \Gamma}{\partial s} \in T_\Gamma^* E$ by Definition \ref{def:tangent_cotangent}.
Then \eqref{eq:hfr_penultimate} leads naturally to the coercivity hypothesis \eqref{eq:contraction_hfr}, which implies that
\begin{align}
    \label{eq:fr_evolve}
    &\le -\lambda \langle \frac{1}{\Gamma}\frac{\partial \Gamma}{\partial s}, \frac{1}{\Gamma}\frac{\partial \Gamma}{\partial s}\rangle_\Gamma = -\lambda \langle \frac{\partial \Gamma}{\partial s}, \frac{\partial \Gamma}{\partial s}\rangle_{{\rm FR}(\Gamma)}.
\end{align}
Hence, by definition \eqref{eq:hfr}, we have
\begin{align}
    d^2_{{\rm HFR}(E)}(\pi_t, \hat{\pi}_t) &\le \int_0^1 \langle \frac{\partial \Gamma}{\partial s},\frac{\partial \Gamma}{\partial s}\rangle_{{\rm FR}(\Gamma)}ds \le e^{-\lambda t}\int_0^1 \langle \frac{\partial \gamma}{\partial s},\frac{\partial\gamma}{\partial s}\rangle_{{\rm FR}(\gamma)}.
\end{align}
Taking the infimum over initial paths $\gamma$ gives the result. 
\end{proof}

We make some observations regarding the coercivity property \eqref{eq:contraction_hfr}.

\begin{remark}
\label{rem:P_coer}
The first term, the $L^2(\pi)$ coercivity $\langle \xi, P_\pi\xi\rangle_\pi = \norm{P_\pi\xi}_{L^2(\pi)}^2 \ge 0$ is guaranteed by the projection property.
As defined, $\ker P_\pi \cap T_\pi^*\Pi(\mu, \cdot)$ is nontrivial, so this is the best we can expect under the general definition \eqref{eq:cotangent}.
However, recalling that the dual variable $h(t, \cdot, \cdot)$ \eqref{eq:sinkhorn_dual} evolves only in the $y$-coordinate (being the right Schr\"odinger potential up to additive $X$-functions), we have that the same property holds for cotangent vectors $\delta h$ along particular solutions, as reflected in its dynamics \eqref{eq:dh_dynamics} which are projected to the space of $Y$-measurable functions by $P_\pi$.
Under this additional structure, which is not captured by Definition \ref{def:tangent_cotangent}, one may expect a nonzero positive coercivity to hold in a manner controlled by initial conditions.
In the following sections, we show via related facts that the entropy decay rate is in fact strictly positive.
\end{remark}

\begin{remark}
The second term in \eqref{eq:fr_evolve}, which arises from time-dependence of the metric, could be eliminated by considering instead the metric induced by the square of the mirror Hessian $(D^2\varphi(\pi))^2$ instead, as is done in the finite-dimensional case (Example 4, \cite{wensing_beyond_2020}).
Using $\delta \pi, \delta h$ to denote tangent and cotangent elements as in \ref{def:tangent_cotangent}, this results in the cotangent unweighted norm
\begin{align}
    \label{eq:metric_2}
    \langle \delta \pi, \delta \pi\rangle_{1/\pi^2} &= \norm{\delta h}_{L^2}^2
\end{align}
which requires that $T_\pi^*E \subset L^2$ rather than only $T_\pi^*E \subset L^2(\pi)$ as in \eqref{eq:cotangent} and could be overly restrictive.
The existence of finite-energy geodesics in the metric \eqref{eq:metric_2} is also not \textit{a priori} guaranteed.
Nevertheless, a formal calculation shows
\begin{align}
    \label{eq:cont_metric_2}
    \frac{d}{dt}\frac12\norm{\delta \pi}_{L^2(1/\pi^2)}^2 &= \frac{d}{dt}\frac12\norm{\delta h}_{L^2}^2 = \langle \delta h,\frac{\partial}{\partial t}\delta h\rangle = -\langle \delta h, D^2 F(\pi)\delta \pi\rangle\\
    &= -\langle \frac{\delta \pi^Y}{\pi^Y}, \left(\frac{\delta \pi}{\pi}\right)^Y\rangle_{L^2(Y)} = -\langle P_{\pi}\frac{\delta \pi}{\pi}, \frac{\delta \pi}{\pi}\rangle_{L^2(X\times Y)}.
\end{align}
However, the required coercivity for exponential decay in this metric \eqref{eq:cont_metric_2},
\begin{align}
    \langle P_\pi \delta h, \delta h\rangle_{L^2(X\times Y)} &\ge \lambda \norm{\delta h}_{L^2}^2
\end{align}
does not hold in general with $\lambda \ge 0$, since for $\pi = \mu \otimes \nu$ and $\delta h = a\otimes b$ with $\E_\nu[b]=0$, which satisfies the tangency constraint \eqref{eq:cotangent}, 
\begin{align}
    \langle P_\pi \delta h, \delta h\rangle &= \iint_{X\times Y}a(x)b(y)^2\E_\mu[a]dxdy = \norm{b}_{L^2(Y)}^2\E_\mu[a]\int_X a(x)dx,
\end{align}
which is negative whenever the $\mu$-weighted and unweighted means of $a \in L^2$ have opposite signs.
Moreover, whereas $P_\pi$ is bounded and self-adjoint in $L^2(\pi)$ \ref{eq:p_adj}, it is not necessarily on the unweighted space; its unweighted $L^2$-adjoint is formally
\begin{align}
    \langle P_\pi f, g\rangle &= \iint_{X\times Y}\frac{g(x, y)}{\pi^Y(y)}dx\int_Xf(x', y)\pi(x', y)dx'dy = \langle f, \pi P_\pi \left[\frac{g}{\pi}\right]\rangle = \langle f, \pi D^2F(\pi) g\rangle.
\end{align} 
\end{remark}

\begin{remark}
In light of the possible coercivity of $P_\pi$ in a restricted subset of $T_\pi^*\Pi(\mu, \cdot)$ as noted in Remark~\ref{rem:P_coer}, a natural question is whether there is a bound $b$ on the second term in \eqref{eq:contraction_hfr} which arises from time-dependence of the metric, as
\begin{align}
    \norm{(I - Q_{\pi})\log\frac{d\pi^Y}{d\nu}}_{L^2(\pi)}^2 \le b.
\end{align}
We show in the next section that this is in fact the exact entropy decay (production) rate of the Sinkhorn flow, a de Bruijn identity.
\end{remark}

\subsection{Entropy decay and gradient flow structure}
\label{sec:thermodynamics}

We give an exact identity for the entropy decay rate of the Sinkhorn flow.
We then show that the flow, viewed via the dynamics of the free marginal $\pi^Y$, is in fact the gradient flow of entropy in a (formal) Riemannian metric.
This reveals that the entropy decay identity is in fact an instance of the entropy-energy (de Bruijn) identity characteristic of diffusion processes.

\begin{theorem}[Entropy decay rate of Sinkhorn]
\label{thm:sinkhorn_epr}
The Sinkhorn flow \eqref{eq:sinkhorn_dual}, \eqref{eq:sinkhorn_primal} is decaying in relative entropy of the right marginal as
\begin{align}
    \label{eq:sinkhorn_epr}
    \frac{d}{dt}H(\pi^Y|\nu) &= \frac{d}{dt}F(\pi) = -\norm{(I - Q_{\pi})\log\frac{d\pi^Y}{d\nu}}_{L^2(\pi)}^2
\end{align}
along solutions $\pi_t \in C^1([0, \infty), L^1)$.
\end{theorem}
\begin{proof}
Let $u(t, y) := \log\frac{d\pi^Y(t, y)}{d\nu(y)}$. 
Since $\pi^Y(t, \cdot)$ is a probability measure,
\begin{align}
\frac{d}{dt} H(\pi^Y|\nu) &= \int_Y \frac{\partial \pi^Y}{\partial t}udy.
\end{align}
From \eqref{eq:dlogpi_dt} and the dual flow \eqref{eq:sinkhorn_dual},
\begin{align}
    \label{eq:dpi_ty_dt}
    \frac{\partial \pi^Y}{\partial t}(t, y) &= \int_X \frac{\partial \pi}{\partial t}(t, x, y)dx = -\int_X[Q_{\pi}^\perp u](t, x, y)\pi(t, x, y)dx.
\end{align}
with differentiation under marginalization justified by the $L^1$ differentiability as in Theorem \ref{thm:cont_fr}.
Hence,
\begin{align}
    \label{eq:dhdt_punchline}
    \frac{d}{dt}H(\pi^Y|\nu) &= -\iint_{X\times Y}u(t, y)(Q_{\pi}^\perp u)(t, x, y)\pi(t, x, y)dxdy = -\norm{Q_{\pi}^\perp u}_{L^2(\pi)}^2
\end{align}
since $Q_{\pi}^\perp$ is an orthogonal projection on $L^2(\pi)$ (Definition~\ref{def:cond_exp}).
\end{proof}

\begin{remark}
Expression \eqref{eq:dhdt_punchline} can be written as, for some function $g \in L^2(\pi^Y)$,
\begin{align}
    \label{eq:L_expansion}
    \langle g, (I - Q_\pi)g\rangle_{\pi} = \langle g, P_\pi(I-Q_\pi)g\rangle_{\pi^Y} = \langle g, (I - P_\pi Q_\pi)g\rangle_{\pi^Y}
\end{align}
    since $P_\pi$ commutes with multiplication by $Y$-measurable functions (and $Q_\pi$ with $X$-measurable functions).
This motivates the definition of the following operator.
\end{remark}

\begin{definition}[Forward-backward conditional expectation]
\label{def:S_pi}
Let $g$ be any $Y$-measurable function.
With $\pi \in {\rm P}(X\times Y)$, let us define
\begin{align}
    \label{eq:S_pi}
    (S_\pi g)(y) &:= (P_\pi Q_\pi g)(y) = \E_\pi[\E_\pi[g(Y)|X]|Y=y],
\end{align}
which maps $L^2(\pi^Y) \to L^2(\pi^Y)$ (compare with the domains of $P_\pi, Q_\pi$ in \ref{def:cond_exp} which are the whole  of $L^2(\pi)$).
$S_\pi$ is self-adjoint on $L^2(\pi^Y)$ since, for $f$ also in $L^2(\pi^Y)$,
\begin{align}
    \langle f, S_\pi g\rangle_{\pi^Y} &= \langle f, Q_\pi g\rangle_{\pi} = \langle Q_\pi f, g\rangle_{\pi} = \langle S_\pi f, g\rangle_{\pi^Y}.
\end{align}
Moreover, $S_\pi \bm{1} = \bm{1}$, hence $S_\pi$ satisfies the necessary conditions to be a symmetric (reversible) Markov operator (\S 1.6.1, \cite{bakry_analysis_2014}).
Its stationary measure is $\pi^Y$, since for every $f \in L^2(\pi^Y)$,
\begin{align}
    \int_Y \pi^Y (S_\pi f) dy &= \int_X \pi^X(Q_\pi f)dx = \int_Y \pi^Y fdy
\end{align}
hence $S_\pi^* \pi^Y = \pi^Y$.
\end{definition}

The dynamics of the right marginal \eqref{eq:dpi_ty_dt} can be written using Definition \ref{def:S_pi} as
\begin{align}
    \label{eq:dpi_ty_2}
    \frac{\partial \pi^Y}{\partial t}  = -\pi^YP_{\pi}(I - Q_{\pi})\log\frac{d\pi^Y}{d\nu} = -\pi^Y(I - S_\pi)\log\frac{d\pi^Y}{d\nu}.
\end{align}
Hence, defining $L_{\pi} := (I - S_\pi)$, \eqref{eq:dpi_ty_2} becomes the \textit{natural gradient} flow \cite{amari_natural_1998} of entropy in the space of right marginals ${\rm P}(Y)$:
\begin{align}
    \label{eq:gradient_flow_entropy}
    \frac{\partial \pi^Y}{\partial t}= -\pi^Y L_{\pi}DH(\pi^Y|\nu) =: -K_\pi DH(\pi^Y|\nu)
\end{align}
in the (formal) Riemannian metric defined by the inverse of the \textit{Onsager operator} $K_\pi := \pi^Y L_{\pi}$.
(Note that constants appearing from $DH$ are annihilated since $S_\pi$ is Markov, hence $L_\pi \bm{1} = 0$.)
Gradient flows of entropy in Riemannian geometries appear ubiquitously in out-of-equilibrium thermodynamics, from the celebrated Wasserstein gradient flow formulation of the Fokker-Planck equation \cite{jordan_variational_1998} to generalized gradient flows arising in other Wasserstein-like metrics \cite{dolbeault_new_2009}.
These geometries capture the physics of the system via the (formal) inverse of the Onsager operator \cite{peletier_variational_2014}, which encodes the Onsager mobility relations.
Abbreviating $\rho := \pi^Y$, in \eqref{eq:gradient_flow_entropy}, $K_\pi : T_{\rho}^*{\rm P}(Y) \to T_{\rho}{\rm P}(Y)$ defines the  inverse of the Riemannian metric along a particular solution $\rho(t, \cdot) \in {\rm P}(Y)$ as 
\begin{align}
    \label{eq:sinkhorn_metric}
    \langle\tau_1, \tau_2\rangle_{T_{\rho}{\rm P}(Y)} &:= \langle \psi_1, K_\pi \psi_2\rangle,\quad K_\pi\psi_i = \tau_i,\quad \psi_i \in T_{\rho}^*{\rm P}(Y).
\end{align}
With $DH(\rho|\nu) \in T_\rho^*{\rm P}(Y)$, this gives the dissipation rate of $H$ along \eqref{eq:gradient_flow_entropy} as
\begin{align}
    \label{eq:F_diss}
    \frac{d}{dt}H(\rho|\nu) &= -\langle DH(\rho|\nu), K_\pi DH(\rho|\nu)\rangle.
\end{align}
This yields the fascinating property that the mirror gradient formulation of the Sinkhorn algorithm is in fact the gradient flow of entropy in its free marginal.

We now note a well-known fact which is that mirror descent can also be expressed as a natural gradient descent \cite{gunasekar_mirrorless_2021}; in the unconstrained finite-dimensional setting,
\begin{align}
    \label{eq:mirror_natural_gradient}
    \frac{d}{dt}h &= \frac{d}{dt}\nabla \varphi(x) = \nabla\nabla\varphi(x)\frac{d}{dt}x = -\nabla F(x)
\end{align}
where $x$ is the primal variable, $h$ the dual variable, and $F, \varphi$ the objective and mirror map as in \eqref{eq:objective}.
However, the entropic gradient structure in \eqref{eq:gradient_flow_entropy} is distinct; the inverse of the Fisher-Rao metric, $\pi$, appears in the definition of $K_\pi$ as expected from \eqref{eq:mirror_natural_gradient}, whereas the nonlocal operator $L_\pi$ \eqref{eq:dpi_ty_2} also appears as a consequence of projection to the right-marginal space ${\rm P}(Y)$.
As noted in Remark \ref{rem:P_coer}, this is the natural space in which to consider dissipation rates since the dual variable $h$ only evolves in the $y$-coordinate \eqref{eq:sinkhorn_dual}.
Indeed, the flow of the right Schr\"odinger potential is just the covariant form of the entropy gradient flow \eqref{eq:gradient_flow_entropy} in the metric \eqref{eq:sinkhorn_metric}:
\begin{align}
    K_\pi\frac{\partial h}{\partial t} &= \frac{\partial \rho}{\partial t},\quad \frac{\partial h}{\partial t} \in T_{\rho}^*{\rm P}(Y),
\end{align}
which we note is also distinct from the metric relating the tangent and cotangent spaces in the original mirror formulation (Definition \ref{def:tangent_cotangent}).
The operator $L_\pi = (I-S_\pi)$
is the generator of a Markov jump process, giving the following physical interpretation of the Sinkhorn flow.

The objective $F$ \eqref{eq:objective} is a Gibbs free energy, since with $\nu = \exp(-V)$ \eqref{eq:densities} (absorbing normalizing constants into $V$), 
\begin{align}
    \label{eq:sinkhorn_fe}
    F(\rho) &= H(\rho|\nu) = \E_{\rho}[V] + H(\rho).
\end{align}
(Note that the temperature-like parameter $\varepsilon$ in the Schr\"odinger bridge energy \eqref{eq:static_sb} does not appear in \eqref{eq:sinkhorn_fe}, and is rather encoded in the initial condition of the flow \eqref{eq:sinkhorn}.)
The \textit{chemical potential}, its first variation (denoted here by $u$ to avoid ambiguity with the marginal constraint $\mu$), evolves by the linear time-inhomogeneous dynamics (from \eqref{eq:gradient_flow_entropy})
\begin{align}
    \label{eq:u_t_dynamics}
    \frac{\partial u}{\partial t} &= -L_\pi u,\quad u := \log\frac{d\rho}{d\nu} \equiv DF(\rho),
\end{align}
with equivalence $(\equiv)$ denoting equality up to additive constants, which as previously noted are in the kernel of $L_\pi$.
Note that by Definition \ref{def:S_pi}, the entropic gradient flow \eqref{eq:gradient_flow_entropy} is in fact the nonlocal diffusion
\begin{align}
    \label{eq:gradient_flow_entropy_2}
    \frac{\partial \rho}{\partial t}(t, y) = -(K_\pi u)(t, y) = \int_Y k_\pi(t, y, y')(u(t, y') - u(t, y))dy',
\end{align}
where $k_\pi$ is the time-inhomogeneous (via the coupling $\pi$) symmetric kernel 
\begin{align}
    \label{eq:kernel}
    k_\pi(t, y, y') &= \int_X\frac{\pi(t, x, y)\pi(t, x, y')}{\pi^X(t, x)}dx,\ \text{with } \int_Y k_\pi(t, y, y')dy' = \rho(t, y).
\end{align}
In this form \eqref{eq:gradient_flow_entropy_2}, it is clear that (nonlocal) differences in the chemical potential drive probability currents, as in linear irreversible thermodynamics \cite{groot_non-equilibrium_2013}.
The entropy production rate is then expressible from \eqref{eq:F_diss} as
\begin{align}
    \label{eq:epr_k}
    -\frac{d}{dt} H(\rho|\nu) &= \langle u, K_\pi u\rangle = \iint_{Y\times Y} u(t, y)k_\pi(t, y, y')(u(t, y) - u(t, y'))dydy'\\
    &= \frac12\int_{Y\times Y} k_\pi(t, y, y')(u(t, y') - u(t, y))^2dydy',
\end{align}
by symmetry of $k_\pi$ \eqref{eq:kernel}.
Hence, again as in linear irreversible thermodynamics (Ch. 4, \cite{groot_non-equilibrium_2013}), the entropy production rate is a quadratic form in the thermodynamic force $(u(t, y) - u(t, y'))$.
The kernel $k_\pi(t, y, y')$ is precisely the Onsager reciprocal relation (constitutive relation giving fluxes from forces) of the Sinkhorn flow, which is why we call $K_\pi$ the Onsager operator.

The Sinkhorn metric is the ``inverse'', in the sense of \eqref{eq:sinkhorn_metric}, of the nonlocal diffusion operator $K_\pi$ just as the Wasserstein metric tensor is the ``inverse'' of the Laplacian $(-\nabla \cdot \rho \nabla) : T_{\rho}{\rm P}^* \to T_{\rho}{\rm P}$.
In fact, \eqref{eq:sinkhorn_metric} is an instance of the \textit{nonlocal Wasserstein metric} proposed and studied in \cite{erbar_gradient_2014} and further developed in \cite{peletier_jump_2022} to identify gradient structures in jump processes.
Its formal inverse is $A_\rho$, where
\begin{align}
    \label{eq:nonlocal_wasserstein_metric}
    (A_\rho u)(y) &:= \int_Y \theta(\rho(y), \rho(y'))(u(y) - u(y'))J(y, dy').
\end{align}
Here, $J$ is the transition kernel of the underlying jump process on the space $Y$ and $\theta$ is a weight, often chosen as the logarithmic mean $\theta(s, t) = (s-t)/(\log s - \log t)$ so that the gradient flow of Shannon entropy ($u = \log \rho$) in this metric gives the standard nonlocal diffusion equation.
Our metric \eqref{eq:sinkhorn_metric} is a special case of \eqref{eq:nonlocal_wasserstein_metric} in which the jump kernel $J$ is time-dependent and the weight $\theta$ is $1$, so that the gradient flow of the Shannon entropy is instead a ($\pi_t$-weighted) nonlocal diffusion equation.

The interpretation of \eqref{eq:gradient_flow_entropy_2} as a weighted diffusion is perhaps most easily seen, as is often done for jump processes \cite{peletier_jump_2022}, by discretization to a complete graph.
Let $G$ be a finite undirected weighted graph $(V, E, k)$ consisting of vertices $y \in V$, edges $(y, y') \in E$, and edge weights $k : E \to \R_+$ (using $y, y'$ to denote vertices in analogy with the spatial coordinates in \eqref{eq:gradient_flow_entropy_2}).
Recall that the graph Laplacian $\Delta_{G}$ of a vertex function $f: V \to \R$ is 
\begin{align}
    \label{eq:graph_laplacian}
    (\Delta_{G} f)(y) &= \sum_{(y, y') \in E} k(y, y')(f(y') - f(y)).
\end{align}
Hence, for a given vertex set $V$, \eqref{eq:gradient_flow_entropy_2} discretizes to the complete graph $E = V^2$ (ignoring self-loops) with time-varying edge weights $k := k_\pi$ given by \eqref{eq:kernel} as 
\begin{align}
    \label{eq:graph_discretization}
    \frac{\partial \rho}{\partial t} &= \Delta_{G} u.
\end{align}
In the ``topology'' $E=V^2$, forces are indeed generated by gradients of the chemical potential $u$.
From this point of view, \eqref{eq:gradient_flow_entropy_2} bears a strong resemblance to classical models of all-to-all coupled oscillator networks \cite{ihara_continuum_2023}.

It is known that nonlocal diffusions share many properties with true (local) diffusions, with the notable exception of regularization of initial conditions \cite{chasseigne_asymptotic_2006}. 
We next derive several such properties for the entropic gradient flow \eqref{eq:gradient_flow_entropy_2} in the following section, with implications for the convergence rate of the Sinkhorn flow.
One immediate property is that\textemdash from a physical point of view\textemdash neither the discrete diffusion \eqref{eq:graph_discretization} nor the continuum nonlocal diffusion \eqref{eq:gradient_flow_entropy_2} can sustain differences in the chemical potential at stationary states (so-called \textit{non-equilibrium steady states}, in which the probability current can be nonzero yet divergence-free).
In the discrete setting, this is an immediate consequence of the spectrum of the (connected) graph Laplacian $\Delta_G$.
In the continuum setting, this follows from the entropy production identity \eqref{eq:epr_k}, which shows that at equilibrium $t_* \in [0, \infty]$, $u(t_*, y) = u(t_*, y')$ a.e. and the thermodynamic force vanishes.
A natural follow-up question is whether there is an exponential rate of contraction of the linear dynamics $u(t, y)$ \eqref{eq:u_t_dynamics} to this set of equilibria, the subspace of constant functions.

Answering this question suggests that the entropic gradient structure \eqref{eq:gradient_flow_entropy} offers more than philosophical or interpretive value over the natural gradient formulation of mirror flow \eqref{eq:mirror_natural_gradient} in the Fisher-Rao metric: the nonlocal diffusive nature of \eqref{eq:sinkhorn_metric} makes the Sinkhorn flow amenable to Bakry-\'Emery-type analysis as in the geometric theory of diffusion processes \cite{bakry_analysis_2014}.
Indeed, the dissipation metric \eqref{eq:F_diss} suggests the state-dependent norm:
\begin{align}
    \label{eq:u_t_contraction}
    \frac{d}{dt}\frac12\norm{u}_{L^2(\rho)}^2 &= - \langle u, L_\pi u\rangle_{\rho} - \frac12\langle u^2, L_\pi u\rangle_{\rho},
\end{align}
(up to the equilibrium set of constants, i.e. further restricting to the mean-zero subspace $L_0^2(\rho)$).
The similarity to diffusion processes, which we explore in depth in the next section, begins with the self-adjointness of $L_\pi$ in this inner product (reflected also in the symmetry of the kernel $k_\pi$ \eqref{eq:kernel}): the first term is recognizable as a Dirichlet-type functional, called the \textit{Dirichlet form} in diffusion \cite{bakry_analysis_2014}.
A Poincar\'e inequality for the Dirichlet form (corresponding to a spectral gap of $L_\pi$) implies exponential contraction in $L^2(\rho)$\textemdash as in diffusion\textemdash if one can control the second term, which arises from the time-inhomogeneity of the metric.

\subsection{Geometry of the Sinkhorn flow}
\label{sec:diffusion}

Having shown that the Sinkhorn flow has the out-of-equilibrium thermodynamic structure as a gradient flow of entropy \eqref{eq:gradient_flow_entropy} in the Riemannian metric defined by a nonlocal diffusion operator \eqref{eq:gradient_flow_entropy_2}, we now show that several key properties characterizing the geometry of (local) diffusion processes \cite{bakry_analysis_2014} are preserved in this entropic gradient flow.

\begin{definition}[Sinkhorn Dirichlet form]
\label{def:dirichlet_form}
In analogy with diffusion, let us define the Dirichlet form (\S 1.7.1, \cite{bakry_analysis_2014}) associated to the Onsager operator $K_\pi$ \eqref{eq:sinkhorn_metric} using an ``integration by parts'' formula
\begin{align}
    \label{eq:dirichlet}
    \mathcal{E}_\pi(f, g) &:= \langle f, K_\pi g\rangle = \langle f, L_\pi g\rangle_{\pi^Y},
\end{align}
which for equal arguments is just, using \eqref{eq:L_expansion},
\begin{align}
    \label{eq:dirichlet_2}
    \mathcal{E}_\pi(u, u) &= \langle u, (I - P_\pi Q_\pi)u\rangle_{\pi^Y} = \langle u, (I - Q_\pi)u\rangle_{\pi} = \norm{(I - Q_\pi)u}_{L^2(\pi)}^2,
\end{align}
the entropy decay rate of the Sinkhorn flow (Theorem \ref{thm:sinkhorn_epr}), taking $u = \log\frac{d\pi^Y}{d\nu}$.
\end{definition}

As in diffusion, we obtain an explicit bound for the entropy decay rate \eqref{eq:sinkhorn_epr} via the ``Poincar\'e constant,'' or spectral gap, of the Dirichlet form $\mathcal{E}_\pi$. 

\begin{lemma}[Poincar\'e inequality for $\mathcal{E}$]
\label{lem:poincare}
For all $u \in L^2(\pi^Y)$, we have
\begin{align}
    \label{eq:poincare_E}
    \mathcal{E}_\pi(u, u) &\ge (1 - C(\pi))\norm{u - \E_{\pi^Y}[u]}_{L^2(\pi^Y)}^2 = (1 - C(\pi))\Var_{\pi^Y}(u)
\end{align}
for some constant $C(\pi) \in [0, 1]$ depending only on $\pi$.
\end{lemma}
\begin{proof}
Let $u \in L^2(\pi^Y)$ and $u = \Tilde{u} + \E_{\pi^Y}[u]\bm{1}$ with $\Tilde{u} \in L^2_0(\pi^Y)$ where $L^2_0(\pi^Y)$ is the mean-zero subspace defined in \eqref{eq:mean_zero}.
Then, since $L_\pi$ is self-adjoint on $L^2(\pi^Y)$ and $S_\pi \bm{1} = \bm{1}$ hence $L_\pi\bm{1} = 0$,
\begin{align}
    \mathcal{E}_\pi(u, u) &= \langle u, L_\pi u\rangle_{\pi^Y} = \langle \Tilde{u}, L_\pi \Tilde{u}\rangle_{\pi^Y} = \norm{\Tilde{u}}_{L^2(\pi^Y)}^2 - \langle \Tilde{u}, S_\pi \Tilde{u}\rangle_{\pi^Y}.
\end{align}
Moreover, $L_0^2(\pi^Y)$ is an invariant subspace of $S_\pi$, since
\begin{align}
    \E_{\pi^Y}[S_\pi \Tilde{u}] &= \langle \Tilde{u}, S_\pi\bm{1}\rangle_{\pi^Y} = \E_{\pi^Y}[\Tilde{u}] = 0.
\end{align}
Therefore, denote the restriction to $L_0^2(\pi^Y)$ by $S_\pi^0$, which remains self-adjoint. 
Letting $E := L_0^2(\pi^Y)$ be the mean-zero subspace, 
\begin{align}
    \label{eq:T0_norm}
    \norm{S_\pi^0} &= \sup W_{E}(S_\pi^0) = \sup_{f \in E\setminus\{0\}}\frac{\langle f, S_\pi^0 f\rangle_{\pi^Y}}{\norm{f}_{L^2(\pi^Y)}^2} = \sup_{f\in E\setminus\{0\}}\frac{\norm{Q_\pi f}_{L^2(\pi)}^2}{\norm{f}_{L^2(\pi)}^2} =: C(\pi).
\end{align}
Since $Q_\pi$ is an orthogonal projection, $C(\pi) \le 1$.
Thus,
\begin{align}
    \mathcal{E}_\pi(u, u) &\ge (1 - C(\pi)) \norm{\Tilde{u}}_{L^2(\pi^Y)}^2,
\end{align}
which gives the result.
\end{proof}

\begin{corollary}
\label{cor:sinkhorn_epr_poincare}
Along the Sinkhorn dynamics \eqref{eq:sinkhorn_dual}, \eqref{eq:sinkhorn_primal} with same regularity as in Theorem \ref{thm:sinkhorn_epr},
\begin{align}
    \label{eq:sinkhorn_epr_poincare}
    \frac{d}{dt} H(\pi^Y|\nu) & \le -(1 - C(\pi))\Var_{\pi^Y}(\log\frac{d\pi^Y}{d\nu})
\end{align}
with $C(\pi_t)$ the Poincar\'e constant in Lemma \ref{lem:poincare} along the flow $\pi_t$.
Furthermore, the dissipation rate $\frac{d}{dt}H(\pi^Y | \nu) < 0$ is nonzero if $\pi_t^Y \ne \nu$ a.e. and the regularization constant $\varepsilon > 0$ is positive.
\end{corollary}
\begin{proof}
The first statement \eqref{eq:sinkhorn_epr_poincare} follows immediately from Theorem \ref{thm:sinkhorn_epr} and Lemma \ref{lem:poincare}.
For the second, suppose that at some time $t$ along the flow $\frac{d}{dt} H(\pi_t^Y|\nu) = 0$ (dropping the $t$ subscripts henceforth).
Then, with $u = \log\frac{d\pi^Y}{d\nu}$ 
\begin{align}
    0 &= -\frac{d}{dt}H(\pi^Y|\nu) = \mathcal{E}_\pi(u, u) = \norm{(I - Q_\pi)u}_{L^2(\pi)}^2
\end{align}
since $Q_\pi$ is an orthogonal projection of $L^2(\pi)$.
This holds iff $u = Q_\pi u$ $\pi$-a.e.
In other words, the $X$-measurable function $\ell(x) := (Q_\pi u)(x)$ is such that $u(y) = \ell(x)$ for $\pi$-a.e. $(x, y)$.

Since $\pi_0 = \mu\otimes \nu > 0$ a.e. and $\varepsilon > 0$ by hypothesis, $\pi_t > 0$ a.e. along the flow \eqref{eq:sinkhorn_primal}.
Hence, $u$ must be constant a.e. on $Y$; since $u = \log\frac{d\pi^Y}{d\nu}$, we have $\pi^Y = \nu$ $Y$-a.e.
\end{proof}

In Corollary \ref{cor:sinkhorn_epr_poincare}, the Poincar\'e inequality (Lemma \ref{lem:poincare}) gives a quantitative bound on the dissipation rate, which can still be zero. 
With positive entropy regularization $\varepsilon > 0$, however, we obtain that the dissipation rate is strictly positive along the flow. 
While the latter is non-quantitative, we can make steps towards a uniform-in-time bound on the Poincar\'e constant $(1 - C(\pi))$ by decoupling the parameter $\pi$ and the argument $u$ of the Dirichlet form $\mathcal{E}_\pi(u, u)$ \eqref{eq:dirichlet}, which are normally related by $u = \log\frac{d\pi^Y}{d\nu}$ along the flow.
Estimating $\mathcal{E}_\pi$ is then possible by assuming a square-integrability of the kernel $k_\pi$ \eqref{eq:kernel}.

\begin{lemma}[Nonzero Poincar\'e constant]
\label{lem:nonzero_poincare}
For arbitrary $u \in L^2(\pi^Y)$, the Poincar\'e constant $(1 - C(\pi)) > 0$ if $\pi > 0$ a.e. and 
\begin{align}
    \label{eq:k_integrable}
    \iint_{Y\times Y}\frac{k_\pi(y, y')^2}{\pi^Y(y)\pi^Y(y')}dydy' &< \infty 
\end{align}
with the kernel $k_\pi$ defined as in \eqref{eq:kernel}.
\end{lemma}
\begin{proof}
Recall from the proof of Lemma \ref{lem:poincare} that $C(\pi) = \norm{S_\pi^0}_{L_0^2(\pi^Y)}$, where $S_\pi^0$ is the restriction of $S_\pi$ to the mean-zero subspace $L_0^2(\pi^Y)$.
Next, note that for a function $u \in L^2(\pi^Y)$, we have by definition of the kernel \eqref{eq:kernel} that 
\begin{align}
    (S_\pi u)(y) &= \int_Y \Tilde{k}_\pi(y, y')u(y')\pi^Y(y'),\quad \Tilde{k}_\pi(y, y') := \frac{k_\pi(y, y')}{\pi^Y(y)\pi^Y(y')}.
\end{align}
Hence, by the integrability condition \eqref{eq:k_integrable},
\begin{align}
    \iint_{Y\times Y}|\Tilde{k}_\pi(y, y')|^2\pi^Y(y)\pi^Y(y')dydy' &< \infty,
\end{align}
so $S_\pi$ is a Hilbert-Schmidt operator on $L^2(\pi^Y)$ and is hence compact, as is the restriction $S_\pi^0$.

Thus, if $C(\pi) = \norm{S_\pi^0}_{L_0^2(\pi^Y)} = 1$, the operator norm is attained for some $f_* \in L_0^2(\pi^Y)\setminus \{0\}$, i.e. from \eqref{eq:T0_norm}
\begin{align}
    \norm{Q_\pi f_*}_{L^2(\pi)}^2 = \norm{f_*}_{L^2(\pi^Y)}^2 = \norm{f_*}_{L^2(\pi)}^2 \iff \norm{(I - Q_\pi)f_*}_{L^2(\pi)}^2 = 0.
\end{align}
Then by an identical argument as in Corollary \ref{cor:sinkhorn_epr_poincare}, since $\pi > 0$ a.e., $f_*$ must be a.e. constant on $Y$, yet constant functions in $L_0^2(\pi^Y)$ are necessarily zero, which is a contradiction.
This gives the result.
\end{proof}

With Lemma \ref{lem:poincare} in hand, a second nonnegativity property \eqref{eq:u_t_contraction} of the Dirichlet form \eqref{eq:dirichlet} (as we will see later, essentially a third moment bound) then implies exponential contraction of the chemical potential $u$ to the subspace of constant functions, as promised in \S \ref{sec:thermodynamics}.

\begin{corollary}
\label{cor:u_contraction}
The Sinkhorn flow, expressed in dynamics of the chemical potential $u$ \eqref{eq:u_t_dynamics}, is contracting with time-dependent rate equal to the Poincar\'e constant \eqref{eq:poincare_E} in the mean-zero subspace $L_0^2(\pi^Y)$ if 
\begin{align}
    \label{eq:contraction_condition}
    \mathcal{E}_{\pi}(\Tilde{u}^2, \Tilde{u}) &\ge 0\quad \forall \pi \in \Pi(\mu, \cdot),
\end{align}
where $u = \log\frac{d\pi^Y}{d\nu}$ is the chemical potential and $\Tilde{u} = u - \E_{\pi^Y}[u]$ its mean-zero part.
\end{corollary}
\begin{proof}
For brevity, let $R_\pi : L^2(\pi^Y) \to L^2_0(\pi^Y)$ be the orthogonal projection 
\begin{align}
    \label{eq:R_pi}
    R_\pi u &:= u - \E_{\pi^Y}[u]\bm{1},
\end{align}
from which it is clear that $R_\pi$ is self-adjoint on $L^2(\pi^Y)$.
Now, let $u(t, \cdot)$ be a particular solution of the chemical potential dynamics \eqref{eq:u_t_dynamics}, and $R_{\pi(t, \cdot, \cdot)}$ the corresponding sequence of projections.
In the time-varying norm defined by $R_\pi$, 
\begin{align}
    \frac{d}{dt}\frac12\norm{R_\pi u}_{\pi^Y}^2 &= \langle R_\pi u, \frac{\partial u}{\partial t}\rangle_{\pi^Y} - \frac{d}{dt}[\E_{\pi^Y}[u]]\langle R_\pi u, \bm{1}\rangle_{\pi^Y} + \frac12\langle R_\pi u, \frac{\partial \pi^Y}{\partial t}R_\pi u\rangle.
\end{align}
The middle term vanishes by \eqref{eq:R_pi}.
Then, since $L_\pi \bm{1} = 0$ and hence $L_\pi = L_\pi R_\pi$,
\begin{align}
     &= -\langle R_\pi u, L_\pi u\rangle_{\pi^Y} - \frac12\langle (R_\pi u)^2, L_\pi u\rangle_{\pi^Y}\\
    &= -\mathcal{E}_{\pi}(R_\pi u, R_\pi u) - \frac12 \mathcal{E}_{\pi}((R_\pi u)^2, R_\pi u).
\end{align}
Then by the Poincar\'e inequality (Lemma~\ref{lem:poincare}) and hypothesis \eqref{eq:contraction_condition},
\begin{align}
    &\le -\lambda_\pi \norm{R_\pi u}_{\pi^Y}^2,
\end{align}
with $\lambda_\pi := 1 - C(\pi)$.
Since $u(t, \cdot)$ is a linear time-varying dynamics \eqref{eq:u_t_dynamics}, this implies that it is contracting in $L_0^2(\pi^Y)$ with rate $\lambda_\pi$.
\end{proof}

Having examined criteria for exponential contraction in various spaces, we finally give a sharp condition for uniform exponential convergence in entropy.
Notice that $\mathcal{E}_\pi(u, u)$ in \eqref{eq:dirichlet_2} with argument $u = \log\frac{d\pi^Y}{d\nu}$ is precisely the  \textit{Fisher information} \cite{amari_information_2016} functional (by which we mean not the usual information-theoretic quantity, but rather a structurally identical one with the Laplacian $-\nabla \cdot \pi \nabla$ replaced with the nonlocal diffusion operator $K_\pi$ \eqref{eq:gradient_flow_entropy_2}):
\begin{align}
    \label{eq:fisher}
    I_\pi(\omega|\nu) &:= \mathcal{E}_\pi(\log\frac{d\omega}{d\nu},\log\frac{d\omega}{d\nu}).
\end{align}
Hence, Theorem \ref{thm:sinkhorn_epr} is exactly de Bruijn's identity\textemdash familiar from diffusion processes (Proposition 5.2.2, \cite{bakry_analysis_2014})\textemdash for the Dirichlet form associated to the Sinkhorn flow:
\begin{align}
    \label{eq:de_bruijn}
    \frac{d}{dt}H(\pi^Y|\nu) = -I_{\pi}(\pi^Y|\nu).
\end{align}
It follows that the entropy decay \eqref{eq:sinkhorn_epr} is uniformly exponential if and only if a log-Sobolev inequality (in the general sense for Markov triples, \S 5.2.1 \cite{bakry_analysis_2014}\textemdash defined by the Dirichlet form \eqref{eq:dirichlet} rather than a true Sobolev inner product) 
holds:
\begin{definition}[Logarithmic Sobolev inequality]
\label{def:lsi}
A pair $(\pi, \nu)$ is said to satisfy a \textit{log-Sobolev inequality} (in the sense of \eqref{eq:fisher}, \eqref{eq:dirichlet}) with constant $\lambda > 0$ if
\begin{align}
    \label{eq:lsi}
    H(\pi^Y|\nu) &\le \frac{1}{2\lambda}I_\pi(\pi^Y|\nu).
\end{align}
\end{definition}
\begin{corollary}[Exponential entropy decay in the Sinkhorn flow]
\label{cor:lsi_epr}
If for given $\mu \in {\rm P}(X), \nu \in {\rm P}(Y), \lambda > 0$ and all $\pi_t$ along the Sinkhorn flow \eqref{eq:sinkhorn_dual}, \eqref{eq:sinkhorn_primal}, the pair $(\pi_t, \nu)$ satisfies the log-Sobolev inequality (Definition \ref{def:lsi}) uniformly with rate $\lambda$, then
\begin{align}
    H(\pi^Y_t|\nu) &\le e^{-2\lambda t}H(\pi_0^Y|\nu).
\end{align}
\end{corollary}
\begin{proof}
This follows immediately from the de Bruijn identity \eqref{eq:de_bruijn} and \eqref{eq:lsi}.
\end{proof}

Finally, to complete the analogy with diffusion processes, we provide a sufficient condition for the LSI which mirrors the Bakry-\'Emery theorem (see \S 5.7, \cite{bakry_analysis_2014}).
We begin by noting that the Dirichlet form \eqref{eq:dirichlet} and hence the Fisher information \eqref{eq:fisher} are expressible as integrals of carr\'e du champ \cite{bakry_diffusions_1985} operators.

\begin{definition}[Sinkhorn carr\'e du champ and iterated c.d.c.]
\label{def:cdc}
Let us define the carr\'e du champ operator $\Gamma$ (\S 1.4.2, \cite{bakry_analysis_2014}, distinct from the pushforward of curves denoted by $\Gamma$ in Theorem \ref{thm:cont_fr}) associated to the generator $-L_\pi$ defining the Sinkhorn marginal flow \eqref{eq:gradient_flow_entropy},
\begin{align}
    \label{eq:gamma}
    \Gamma(f, g) &= \frac12[-L_\pi(fg) + fL_\pi g + gL_\pi f],\quad \Gamma(f) := \Gamma(f, f),
\end{align}
dropping the subscript $\pi$ from $\Gamma$ for brevity.
Similarly, we define the iterated carr\'e du champ (\S 1.16, \cite{bakry_analysis_2014}) associated to $-L_\pi$,
\begin{align}
    \label{eq:gamma_2}
    \Gamma_2(f, g) &= \frac12[-L_\pi \Gamma(f, g) + \Gamma(f, L_\pi g) + \Gamma(L_\pi f, g)],\quad \Gamma_2(f) := \Gamma_2(f, f).
\end{align}
Since $L_\pi$ is self-adjoint on $L^2(\pi^Y)$ and $L_\pi \bm{1} = 0$,
\begin{align}
    \label{eq:gamma_dirichlet}
    \int_Y\Gamma(f, g)d\pi^Y &= \int_Y(fL_\pi g)d\pi^Y = \mathcal{E}_\pi(f, g)
\end{align}
which is just the Dirichlet form (Definition \ref{def:dirichlet_form}), verifying the ``integration by parts'' formula for $L_\pi$ (\S 1.7.1, \cite{bakry_analysis_2014}).
Similarly, 
\begin{align}
    \label{eq:gamma_2_int}
    \int_Y\Gamma_2(f, g)d\pi^Y &= \frac12\int_Y (\Gamma(f, L_\pi g) + \Gamma(L_\pi f, g))d\pi^Y = \int_Y (L_\pi f)(L_\pi g)d\pi^Y,
\end{align}
as in diffusion.
\end{definition}

Finally, we have the following ``curvature condition'' (borrowing terminology from diffusion, \S 5.7 \cite{bakry_analysis_2014}) under which the time-varying Dirichlet form \eqref{eq:dirichlet} satisfies a uniform LSI \ref{def:lsi}.

\begin{theorem}[$\Gamma$-$\Gamma_2$ condition for Sinkhorn LSI]
\label{thm:dynamic_curvature}
If, along the Sinkhorn flow \eqref{eq:gradient_flow_entropy}, 
\begin{align}
    \label{eq:dynamic_curvature}
    \Gamma_2(u) \ge \lambda \Gamma(u)\quad \text{where } u = \log\frac{d\pi^Y}{d\nu}
\end{align}
$t$-uniformly for some $\lambda > 0$, and the third moment
\begin{align}
    \label{eq:dynamic_curvature_lb}
    \E_\pi[((I - Q_\pi)u)^3] &\ge 0
\end{align}
is nonnegative, then the Sinkhorn LSI \eqref{eq:lsi} holds along the flow.
\end{theorem}
\begin{proof}
We proceed essentially as in diffusion (5.7.4, \cite{bakry_analysis_2014}).
First, note that $d\pi^Y = e^{u}d\nu$ from \eqref{eq:dynamic_curvature}. 
Hence, from \eqref{eq:gamma_dirichlet}, the Fisher information functional \eqref{eq:fisher} associated to $L_\pi$ is expressible as
\begin{align}
    \label{eq:gamma_fisher}
    I_{\pi}(\pi^Y|\nu) &= \mathcal{E}_\pi(u, u) = \int_Y e^{u}\Gamma(u)d\nu.
\end{align}
Differentiating \eqref{eq:gamma_fisher} in time using \eqref{eq:u_t_dynamics} and \eqref{eq:gamma_2_int}, we have
\begin{align}
    \frac{d}{dt}I_{\pi}(\pi^Y|\nu) &= \int_Ye^u[-(L_{\pi}u)\Gamma(u) + 2\Gamma(u, \frac{\partial u}{\partial t}) + \left(\frac{\partial \Gamma}{\partial t}\right)u]\,d\nu\\
    \label{eq:dI_dt}
    &= -\int_Y e^{u}[(L_{\pi}u)\Gamma(u) + 2\Gamma_2(u) + \frac12 \frac{\partial L_\pi}{\partial t}(u^2) - u\frac{\partial L_{\pi}}{\partial t}u]\,d\nu.
\end{align}
Expanding the first term in \eqref{eq:dI_dt} (dropping the subscripts $\pi$ from operators), we obtain
\begin{align}
    \langle Lu, \Gamma(u)\rangle_{\pi^Y} 
&=\frac12\left[\E_{\pi^Y}[u^3] - \langle Su, 2u^2 - 2uSu + Su^2\rangle_{\pi^Y}\right]
\end{align}
Let us abbreviate $\dot{L} := \partial L/\partial t$ (and similarly $\dot{P}$ \eqref{eq:P_pi} and $\dot{Q}$ \eqref{eq:Q_pi}).
Recalling from \eqref{eq:dlogpi_dt} that $\dot{\pi} = -\pi(I-Q)u$, we have, for $f$ which is $X$-measurable and $g$ which is $Y$-measurable, 
\begin{align}
    \dot{P} f &= (Lu)(Pf) + P(fQu) - P(fu) = -(Su)(Pf) + P(fQu)\\
    \dot{Q} g &= Q(gQu) - Q(gu) = (Qu)(Qg) - Q(gu),
\end{align}
recalling that $P, Q$ commute with multiplication by functions which are $Y, X$-measurable respectively (see earlier \eqref{eq:Q_dot} for a similar identity).
Hence, 
\begin{align}
    \dot{L}u &= -(\dot{P}Q + P\dot{Q})u = (Su)^2 - 2P((Qu)^2) + Su^2\\
    \dot{L}(u^2) &= (Su)(Su^2) - 2P((Qu)(Qu^2)) + Su^3.
\end{align}
Thus, the entirety of the non-$\Gamma_2$ contribution in \eqref{eq:dI_dt} reads, after cancellations using self-adjointness of $S$ on $L^2(\pi^Y)$ and $S\bm{1} = \bm{1}$,
\begin{align}
    &\langle Lu, \Gamma(u)\rangle_{\pi^Y} + \frac12 E_{\pi^Y}[\dot{L}u^2] - \langle u, \dot{L}u\rangle_{\pi^Y}\\
=&\, \E_{\pi^Y}[u^3 - P((Qu)(Qu^2))] + 2\langle u, P((Qu)^2) - Su^2\rangle_{\pi^Y}.
\end{align}
By definition of $P$,
\begin{align}
    &= \E_\pi[u^3 - (Qu)(Qu^2) + 2u(Qu)^2 - 2u(Qu^2)]
\end{align}
and conditioning inside on $X$, using that $Q$ commutes with multiplication by $X$-measurable functions,
\begin{align}
    &= \E_\pi[u^3 - 3(Qu)(Qu^2) + 2(Qu)^3]\\
    \label{eq:third_moment}
    &= \E_\pi[((I - Q)u)^3],
\end{align}
which is the third moment \eqref{eq:dynamic_curvature_lb}.
Hence, from \eqref{eq:dI_dt} we have the property that the second moment evolves by the third moment plus a $\Gamma_2$ contribution,
\begin{align}
    \label{eq:moment_evolution}
    \frac{d}{dt}I_\pi(\pi^Y|\nu) &= \frac{d}{dt}\E_\pi[((I - Q_\pi)u)^2] = -\E_\pi[((I - Q_\pi)u)^3] -2\int_Y e^u\Gamma_2(u)d\nu.
\end{align}
Then in particular, the hypotheses \eqref{eq:dynamic_curvature}, \eqref{eq:dynamic_curvature_lb} and the Fisher identity \eqref{eq:gamma_fisher} imply
\begin{align}
    \frac{d}{dt}I_{\pi}(\pi^Y|\nu) &\le -2\lambda \int_Y e^u \Gamma(u)d\nu = -2\lambda I_{\pi}(\pi^Y|\nu),
\end{align}
thus the entropy dissipation rate is itself decaying exponentially with rate $2\lambda$ (as in diffusion \cite{markowich_trend_nodate}). 
Integrating de Bruijn's identity \eqref{eq:de_bruijn} in time, we have
\begin{align}
    H(\pi_t^Y|\nu) &= \int_t^\infty I_{\pi_s}(\pi_s^Y|\nu)ds \le I_{\pi_t}(\pi_t^Y|\nu)\int_0^\infty e^{-2\lambda s}ds = \frac{1}{2\lambda}I_{\pi_t}(\pi_t^Y|\nu),
\end{align}
which gives the result.
\end{proof}

\begin{remark}
\label{rem:moments}
The identity \eqref{eq:moment_evolution} suggests a possible ``moment hierarchy'': the first moment 
\begin{align}
    \E_\pi[(I - Q_\pi)u] &= \E_\pi[u] - \langle Q_\pi\bm{1}, u\rangle_\pi = 0
\end{align}
is fixed, and the evolution of the second has a contribution from the third.
Furthermore, the contraction condition \eqref{eq:contraction_condition} for the chemical potential dynamics in Corollary \ref{cor:u_contraction} is closely related to the bound \eqref{eq:dynamic_curvature_lb}; letting $v = (I - Q_\pi)u = (I - Q_\pi)\Tilde{u}$ (with $\Tilde{u} = u - \E_{\pi^Y}[u]$ the mean-zero part as in \eqref{eq:contraction_condition}),
\begin{align}
    \mathcal{E}_\pi(\Tilde{u}^2, \Tilde{u}) &=\E_\pi[\Tilde{u}^2v]
= \E_\pi[v^3] + 2\E_\pi[v^2Q_\pi\Tilde{u}]
\end{align}
which is just the third moment \eqref{eq:third_moment} plus a correction $2\E_\pi[((I - Q_\pi)\Tilde{u})^2Q_\pi \Tilde{u}]$.
\end{remark}

Let us conclude this section by noting that, in the case of true (local) diffusions, the $\Gamma-\Gamma_2$ relationship \eqref{eq:dynamic_curvature} is a differential condition called the \textit{curvature-dimension condition} or ${\rm CD}(\lambda, \infty)$ (with the second parameter playing a role not discussed here, see 3.3.14, \cite{bakry_analysis_2014}).
More specifically, when $\Gamma^A, \Gamma_2^A$ are the carr\'e du champ operators defined as in \ref{def:cdc} but instead for the drift-diffusion generator $A = \Delta - \nabla U \cdot \nabla$ for some potential $U$ (working on flat $\R^d$), ${\rm CD}(\lambda, \infty)$ becomes (\S 1.16, \cite{bakry_analysis_2014})
\begin{align}
    \Gamma_2^A(f) &\ge \lambda \Gamma^A(f) \iff \magn{\nabla\nabla f}^2 + \langle \nabla f,\nabla\nabla U \nabla  f\rangle \ge \lambda \magn{\nabla f}^2
\end{align}
which is satisfied by the strict convexity $\nabla\nabla U \ge \lambda I$, hence the term ``curvature condition.''
This is the starting point of the Bakry-\'Emery theorem (\S5, \cite{markowich_trend_nodate}) which shows that this implies the (classical) logarithmic Sobolev inequality for the measure $e^{-U}$.

On the other hand, here \eqref{eq:dynamic_curvature} is intrinsically an integral condition, reflecting the behavior of conditional expectation operators that ultimately arise from the metric \eqref{eq:gradient_flow_entropy_2} defining the entropic gradient flow.
Remarkably, however, up to control of time-inhomogeneous terms arising in various guises \eqref{eq:contraction_condition}, \eqref{eq:dynamic_curvature_lb} (which, in light of Remark \ref{rem:moments}, are all essentially cubic moments), many of the key geometric features governing diffusions\textemdash from the de Bruijn identity \eqref{eq:de_bruijn} to the curvature-like condition \eqref{eq:dynamic_curvature}\textemdash are preserved in the nonlocal diffusive structure \eqref{eq:gradient_flow_entropy_2} of the Sinkhorn flow's free-marginal dynamics.
By studying the Dirichlet form \eqref{eq:dirichlet} in its own right, without coupling the parameter $\pi$ and argument $u$ along the flow (as done in Lemma \ref{lem:nonzero_poincare}), we may understand these properties as more generic features of nonlocal diffusion generators which are preserved from the local case \cite{chasseigne_asymptotic_2006}.

\section{Applications}

We now give for illustration two straightforward computational use-cases in which the Sinkhorn log-Sobolev constant $\lambda$ is known or can be estimated \textit{a priori}.
This may be estimated in the discrete case by direct computation of \eqref{eq:lsi} over a binned feasible set, or via Theorem \ref{thm:dynamic_curvature}, which essentially reduces one required integration.

\begin{example}[Latent space design for generative models]
In generative models trained using ${\rm OT}_\varepsilon$-type losses (e.g. the Sinkhorn divergence, \cite{genevay_learning_2018}), the choice of the latent space $\phi(Y)$ (which $Y$ the space of the data distribution or the denoising target) for training can be guided by the LSI constant \eqref{eq:lsi}, e.g. by additional loss terms, of the pushforward data marginal $\phi_\#\nu$, since larger (uniform) LSI constants yield faster exponential convergence of inner Sinkhorn solves.
Similarly, in generative models based upon the Schr\"odinger bridge (e.g. \cite{de_bortoli_diffusion_2021}), a larger positive LSI constant in the latent space $\phi(Y)$ can improve training stability and convergence rates.
This could be achieved by differentiating a suitably regular surrogate for the ratio \eqref{eq:lsi} or \eqref{eq:dynamic_curvature} with respect to parameters of the latent space mapping for a given data distribution and adding an appropriate loss term.
\end{example}

\begin{example}[Adaptive stopping]
Practical uses of the Sinkhorn algorithm often use a fixed number $L$ of iterations; we illustrate how \textit{a priori} bounds for the entropy decay rate can be used to set $L$.  
While the identity \eqref{eq:sinkhorn_epr} holds for the continuous-time Sinkhorn flow, the entropy drop across the $\gamma$-Sinkhorn iteration \eqref{eq:gamma_sinkhorn_md} is first-order consistent with \eqref{eq:sinkhorn_epr}:
\begin{align}
    \frac{H(\pi_{t+1}^Y|\nu) - H(\pi_{t}^Y|\nu)}{\gamma} &= \frac{d}{dt}H(\pi_t^Y|\nu) + O(\gamma)
\end{align}
for step size $\gamma > 0$ (where $t \in \N$ correspond to full steps of the Sinkhorn algorithm \eqref{eq:sinkhorn}).
If one has a uniform LSI \ref{def:lsi} of rate $\lambda$ for the marginal $\nu$ (say, in a well-designed latent space), then 
\begin{align}
    H(\pi_{t+k}^Y|\nu)  &\le e^{-2\lambda \gamma k}H(\pi_{t}^Y|\nu) + O(k\gamma^2)
\end{align}
(which can also be adapted for variable step-sizes).
Hence for given tolerance $q > 0$ and error $H_0$ measured at some $t$, one can plan for
\begin{align}
    \label{eq:n_iterates}
    k &\ge \left\lceil\frac{1}{2\lambda\gamma}\log\frac{H_0}{q}\right\rceil
\end{align}
further iterates, at which point $H_0$ can be re-measured and checked for within tolerance, else the iteration re-started with a new estimate for $k$.
We note that as classical Sinkhorn \eqref{eq:md_sinkhorn} corresponds to $\gamma =1$, \eqref{eq:n_iterates} is merely a heuristic to avoid computing $H$ on every step, providing a valid estimate when using the $\gamma$-Sinkhorn iteration \eqref{eq:gamma_sinkhorn_md} with $\gamma \ll 1$.
Similarly, in Sinkformers \cite{sander_sinkformers_2022}, normalization of the attention kernel in a transformer architecture is generalized from a single softmax step to an iterative bi-normalization using some fixed number of Sinkhorn iterations.
This is a key hyperparameter of the model; in \cite{sander_sinkformers_2022}, it is noted that the number of iterations can influence the model accuracy.
Using an adaptively estimated number of $\gamma$-Sinkhorn iterations as in \eqref{eq:n_iterates} can account for the kernel's varying geometry both over time and across attention heads.
\end{example}

\section{Discussion}
Each full step of the Sinkhorn algorithm \eqref{eq:sinkhorn} 
admits the following interpretation in terms of the original Schr\"odinger bridge problem \eqref{eq:schr_question}: with the left marginal or initial observation of the bridge $\mu$ held fixed along the mirror descent \eqref{eq:sinkhorn_primal}, the right marginal $\nu_t := \pi_t^Y$ relaxes to the imposed terminal observation or constraint $\nu$.
By the Poincar\'e inequality (Corollary \ref{cor:sinkhorn_epr_poincare}) and Pinsker's inequality (e.g. Lemma 1.2 \cite{nutz_introduction_nodate}), this converges in total variation as
\begin{align}
    \norm{\nu_t - \nu}_{\rm TV} &\le \sqrt{2H(\nu_t|\nu)} \overset{t\to\infty}{\to} 0
\end{align}
and hence narrowly as in \eqref{eq:schr_question}.
The ${\rm TV}$ decay is furthermore exponential if the log-Sobolev inequality (Corollary \ref{cor:lsi_epr}) holds.

In other words, the act of conditioning on the large deviation in Schr\"odinger's original problem \eqref{eq:schr_question} can be achieved by (or endowed with the structure of) the stochastic dynamics of a particle whose law is the gradient flow of entropy in the nonlocal diffusion metric \eqref{eq:gradient_flow_entropy_2}.
As noted earlier, this structure is highly reminiscent of the nonlocal Wasserstein \cite{erbar_gradient_2014,peletier_jump_2022} and other generalized gradient flows \cite{peletier_variational_2014} appearing in out-of-equilibrium thermodynamics.
It suggests a physical interpretation: taking physical time to be algorithmic time $t$ along the Sinkhorn flow \eqref{eq:sinkhorn_dual}, \eqref{eq:sinkhorn_primal} rather than $\tau$ along the bridge \eqref{eq:dynamic_sb}, a system which relaxes to equilibrium by probability currents induced by weighted \eqref{eq:kernel} nonlocal differences in a chemical potential solves the entropic optimal transport problem \eqref{eq:eot_intro}.
The inverse metric tensor \eqref{eq:gradient_flow_entropy_2}, or \textit{mobility relation} of this dissipative system, is similar to classical models of synchronization via all-all coupling such as the continuum Kuramoto model \cite{kuramoto_chemical_1984}.
This, supplemented by the many geometric features the flow shares with true diffusion processes as shown in \S \ref{sec:diffusion}, raises the intriguing possibility of natural systems which inherently solve entropic optimal transport or Schr\"odinger bridge problems. 
An immediate application of relevance would be a physical system which, by realizing this gradient flow, solves the Schr\"odinger bridge problem between given noise and data distributions, in principle training a diffusion generative model \cite{de_bortoli_diffusion_2021,albergo_stochastic_2023}.

Outside of computational problems, examples of transport optimization abound in biological systems, from the famous maze-solving abilities of the slime mold \textit{Physarum polycephalum} \cite{nakagaki_maze-solving_2000} to the recent Waddington optimal transport view of distributional shifts in gene expression space \cite{schiebinger_optimal-transport_2019}.
The latter demonstrates that intermediate expression profiles between snapshots of ancestor and descendent cell populations can in some cases be successfully predicted by (entropic) OT interpolations.
This presents an intriguing biological cousin to Schr\"odinger's problem: if cell states (captured by their \textit{expression profiles}, i.e. per-gene mRNA copy numbers at a given time) are thought of as diffusion-limited particles in gene expression space which present ``surprising'' snapshots as in \eqref{eq:lln_deviation}\textemdash corresponding e.g. to the development of new cell types\textemdash what is the most likely path in cell-state space taken between them?
Furthermore, the logic of this trajectory inference point of view \cite{lavenant_toward_2024} can be reversed to interpret the developmental process itself as a stochastic optimal control problem \cite{chen_stochastic_2021} in the expression space.
The optimal control (expression drift) can then be obtained by physically realizing \eqref{eq:gradient_flow_entropy_2}, with algorithmic time $t$ representing, among many possibilities, developmental or evolutionary parameters which vary slowly relative to the expression (bridge) time $\tau$.

The microscopic details \cite{liero_microscopic_2017} of a physical system which implements the entropic gradient flow \eqref{eq:gradient_flow_entropy_2} remain open; however, many physical models are now known to have gradient structures with free energy and mobility combinations generalizing the classical Wasserstein case, including notably the case of certain reactive flows \cite{mielke_gradient_2011}.
From the analytical side, further work remains to (i) quantitatively relate the Poincar\'e constant \eqref{eq:sinkhorn_epr_poincare} and the entropic regularization parameter $\varepsilon$ and imposed marginals, (ii) establish possible conditions under which uniform-in-time bounds for it hold along the flow, and (iii) interpret the curvature-like condition \eqref{eq:dynamic_curvature} in terms of the imposed marginals $e^{-U}, e^{-V}$.
At the moment, the thermodynamic interpretation of the Sinkhorn flow we offer here complements existing analysis with (i) a toolbox of functional inequalities mirroring those from the geometry of diffusion, (ii) algorithmic stopping and latent space design heuristics, and (iii) a physical entropic gradient flow interpretation in which certain systems driven to equilibrium by nonlocal diffusions may be implicitly solving entropic optimal transport or Schr\"odinger bridge problems.

\bibliography{references}

@inproceedings{karimi_sinkhorn_2024,
	address = {Valencia, Spain},
	title = {Sinkhorn {Flow} as {Mirror} {Flow}: {A} {Continuous}-{Time} {Framework} for {Generalizing} the {Sinkhorn} {Algorithm}},
	issn = {2640-3498},
	shorttitle = {Sinkhorn {Flow} as {Mirror} {Flow}},
	url = {https://proceedings.mlr.press/v238/reza-karimi24a.html},
	language = {en},
	booktitle = {Proceedings of {The} 27th {International} {Conference} on {Artificial} {Intelligence} and {Statistics}},
	publisher = {PMLR},
	author = {Karimi, Mohammad Reza and Hsieh, Ya-Ping and Krause, Andreas},
	month = apr,
	year = {2024},
	pages = {4186--4194},
}

@misc{leonard_survey_2013,
	title = {A survey of the {Schr\"odinger} problem and some of its connections with optimal transport},
	doi = {10.48550/arXiv.1308.0215},
	publisher = {arXiv},
	author = {L{\'e}onard, Christian},
	month = aug,
	year = {2013},
	note = {arXiv:1308.0215 [math]},
}

@inproceedings{de_bortoli_diffusion_2021,
	address = {Red Hook, NY, USA},
	series = {{NIPS} '21},
	title = {Diffusion schrödinger bridge with applications to score-based generative modeling},
	isbn = {978-1-7138-4539-3},
	booktitle = {Proceedings of the 35th {International} {Conference} on {Neural} {Information} {Processing} {Systems}},
	publisher = {Curran Associates Inc.},
	author = {De Bortoli, Valentin and Thornton, James and Heng, Jeremy and Doucet, Arnaud},
	month = dec,
	year = {2021},
	pages = {17695--17709},
}

@inproceedings{cuturi_sinkhorn_2013,
	address = {Red Hook, NY, USA},
	series = {{NIPS}'13},
	title = {Sinkhorn distances: lightspeed computation of optimal transport},
	volume = {2},
	shorttitle = {Sinkhorn distances},
	booktitle = {Proceedings of the 27th {International} {Conference} on {Neural} {Information} {Processing} {Systems} - {Volume} 2},
	publisher = {Curran Associates Inc.},
	author = {Cuturi, Marco},
	month = dec,
	year = {2013},
	pages = {2292--2300},
}

@article{wensing_beyond_2020,
	title = {Beyond convexity---{Contraction} and global convergence of gradient descent},
	volume = {15},
	issn = {1932-6203},
	doi = {10.1371/journal.pone.0236661},
	number = {8},
	journal = {PLOS ONE},
	author = {Wensing, Patrick M. and Slotine, Jean-Jacques},
	month = aug,
	year = {2020},
	pages = {e0236661},
}

@article{lohmiller1998contraction,
  title={On contraction analysis for nonlinear systems},
  author={Lohmiller, Winfried and Slotine, Jean-Jacques E.},
  journal={Automatica},
  volume={34},
  number={6},
  pages={683--696},
  year={1998},
  doi={10.1016/S0005-1098(98)00019-3},
  publisher={Pergamon}
}

@misc{peyre_computational_2020,
	title = {Computational {Optimal} {Transport}},
	doi = {10.48550/arXiv.1803.00567},
	publisher = {arXiv},
	author = {Peyr{\'e}, Gabriel and Cuturi, Marco},
	month = mar,
	year = {2020},
	note = {arXiv:1803.00567 [stat]},
}

@misc{nutz_introduction_nodate,
	title = {Introduction to {Entropic} {Optimal} {Transport}},
	language = {en},
	author = {Nutz, Marcel},
	note = {Lecture notes, Columbia University},
	year = {2022},
}

@article{dolbeault_new_2009,
	title = {A new class of transport distances between measures},
	volume = {34},
	issn = {1432-0835},
	doi = {10.1007/s00526-008-0182-5},
	language = {en},
	number = {2},
	journal = {Calculus of Variations and Partial Differential Equations},
	author = {Dolbeault, Jean and Nazaret, Bruno and Savar{\'e}, Giuseppe},
	month = feb,
	year = {2009},
	pages = {193--231},
}

@inproceedings{genevay_learning_2018,
	address = {Canary Islands},
	title = {Learning {Generative} {Models} with {Sinkhorn} {Divergences}},
	issn = {2640-3498},
	url = {https://proceedings.mlr.press/v84/genevay18a.html},
	booktitle = {Proceedings of the {Twenty}-{First} {International} {Conference} on {Artificial} {Intelligence} and {Statistics}},
	publisher = {PMLR},
	author = {Genevay, Aude and Peyr\'e, Gabriel and Cuturi, Marco},
	month = mar,
	year = {2018},
	pages = {1608--1617},
}

@book{ambrosio_gradient_2008,
	address = {Basel Boston Berlin},
	series = {Lectures in {Mathematics} {ETH} {Zürich}},
	title = {Gradient {Flows}: in {Metric} {Spaces} and in the {Space} of {Probability} {Measures}},
	isbn = {978-3-7643-8722-8},
	shorttitle = {Gradient {Flows}},
	language = {eng},
	publisher = {Birkh{\"a}user},
	author = {Ambrosio, Luigi and Gigli, Nicola and Savar{\'e}, Giuseppe},
	year = {2008},
	doi = {10.1007/978-3-7643-8722-8},
}

@article{amari_natural_1998,
	title = {Natural {Gradient} {Works} {Efficiently} in {Learning}},
	volume = {10},
	issn = {0899-7667},
	doi = {10.1162/089976698300017746},
	number = {2},
	journal = {Neural Computation},
	author = {Amari, Shun-ichi},
	month = feb,
	year = {1998},
	pages = {251--276},
}

@article{markowich_trend_nodate,
	title = {On the trend to equilibrium for the {Fokker}--{Planck} equation: an interplay between physics and functional analysis},
	author = {Markowich, Peter A. and Villani, C{\'e}dric},
	journal = {Mat. Contemp.},
	volume = {19},
	year = {2000},
	pages = {1--29},
}

@article{beck_mirror_2003,
	title = {Mirror descent and nonlinear projected subgradient methods for convex optimization},
	volume = {31},
	issn = {0167-6377},
	doi = {10.1016/S0167-6377(02)00231-6},
	number = {3},
	journal = {Operations Research Letters},
	author = {Beck, Amir and Teboulle, Marc},
	month = may,
	year = {2003},
	pages = {167--175},
}

@book{bakry_analysis_2014,
	address = {Cham},
	title = {Analysis and {Geometry} of {Markov} {Diffusion} {Operators}: 348},
	isbn = {978-3-319-00226-2},
	shorttitle = {Analysis and {Geometry} of {Markov} {Diffusion} {Operators}},
	language = {English},
	publisher = {Springer},
	author = {Bakry, Dominique and Gentil, Ivan and Ledoux, Michel},
	year = {2014},
	doi = {10.1007/978-3-319-00227-9},
}

@inproceedings{aubin-frankowski_mirror_2022,
	address = {Red Hook, NY, USA},
	series = {{NIPS} '22},
	title = {Mirror descent with relative smoothness in measure spaces, with application to sinkhorn and {EM}},
	isbn = {978-1-7138-7108-8},
	booktitle = {Proceedings of the 36th {International} {Conference} on {Neural} {Information} {Processing} {Systems}},
	publisher = {Curran Associates Inc.},
	author = {Aubin-Frankowski, Pierre-Cyril and Korba, Anna and Léger, Flavien},
	month = nov,
	year = {2022},
	pages = {17263--17275},
}

@inproceedings{sander_sinkformers_2022,
	address = {Valencia, Spain},
	title = {Sinkformers: {Transformers} with {Doubly} {Stochastic} {Attention}},
	issn = {2640-3498},
	shorttitle = {Sinkformers},
	url = {https://proceedings.mlr.press/v151/sander22a.html},
	booktitle = {Proceedings of {The} 25th {International} {Conference} on {Artificial} {Intelligence} and {Statistics}},
	publisher = {PMLR},
	author = {Sander, Michael E. and Ablin, Pierre and Blondel, Mathieu and Peyré, Gabriel},
	month = may,
	year = {2022},
	pages = {3515--3530},
}

@book{amari_information_2016,
	address = {Tokyo},
	series = {Applied {Mathematical} {Sciences}},
	title = {Information {Geometry} and {Its} {Applications}},
	volume = {194},
	copyright = {https://www.springernature.com/gp/researchers/text-and-data-mining},
	isbn = {978-4-431-55977-1 978-4-431-55978-8},
	publisher = {Springer Japan},
	author = {Amari, Shun-ichi},
	year = {2016},
	doi = {10.1007/978-4-431-55978-8},
}

@inproceedings{bakry_diffusions_1985,
	address = {Berlin, Heidelberg},
	title = {Diffusions hypercontractives},
	isbn = {978-3-540-39397-9},
	doi = {10.1007/BFb0075847},
	language = {fr},
	booktitle = {S{\'e}minaire de {Probabilit{\'e}s} {XIX} 1983/84},
	publisher = {Springer},
	author = {Bakry, D. and {\'E}mery, M.},
	editor = {Az{\'e}ma, Jacques and Yor, Marc},
	year = {1985},
	pages = {177--206},
}

@misc{peletier_variational_2014,
	title = {Variational {Modelling}: {Energies}, gradient flows, and large deviations},
	shorttitle = {Variational {Modelling}},
	doi = {10.48550/arXiv.1402.1990},
	publisher = {arXiv},
	author = {Peletier, Mark A.},
	month = feb,
	year = {2014},
	note = {arXiv:1402.1990 [math-ph]},
}

@article{jordan_variational_1998,
	title = {The {Variational} {Formulation} of the {Fokker}--{Planck} {Equation}},
	volume = {29},
	issn = {0036-1410},
	doi = {10.1137/S0036141096303359},
	number = {1},
	journal = {SIAM Journal on Mathematical Analysis},
	author = {Jordan, Richard and Kinderlehrer, David and Otto, Felix},
	month = jan,
	year = {1998},
	pages = {1--17},
}

@article{mielke_gradient_2011,
	title = {A gradient structure for reaction--diffusion systems and for energy-drift-diffusion systems},
	volume = {24},
	issn = {0951-7715, 1361-6544},
	doi = {10.1088/0951-7715/24/4/016},
	number = {4},
	journal = {Nonlinearity},
	author = {Mielke, Alexander},
	month = apr,
	year = {2011},
	pages = {1329--1346},
}

@article{chen_stochastic_2021,
	title = {Stochastic {Control} {Liaisons}: {Richard} {Sinkhorn} {Meets} {Gaspard} {Monge} on a {Schr\"odinger} {Bridge}},
	volume = {63},
	issn = {0036-1445},
	shorttitle = {Stochastic {Control} {Liaisons}},
	doi = {10.1137/20M1339982},
	number = {2},
	journal = {SIAM Review},
	author = {Chen, Yongxin and Georgiou, Tryphon T. and Pavon, Michele},
	month = jan,
	year = {2021},
	pages = {249--313},
}

@article{pavon_span_2021,
	title = {The Data‐Driven Schr\"odinger Bridge},
	volume = {74},
	issn = {0010-3640, 1097-0312},
	doi = {10.1002/cpa.21975},
	number = {7},
	journal = {Communications on Pure and Applied Mathematics},
	author = {Pavon, Michele and Trigila, Giulio and Tabak, Esteban G.},
	month = jul,
	year = {2021},
	pages = {1545--1573},
}

@book{dembo_large_2010,
	address = {Berlin, Heidelberg},
	series = {Stochastic {Modelling} and {Applied} {Probability}},
	title = {Large {Deviations} {Techniques} and {Applications}},
	volume = {38},
	copyright = {http://www.springer.com/tdm},
	isbn = {978-3-642-03310-0 978-3-642-03311-7},
	publisher = {Springer},
	author = {Dembo, Amir and Zeitouni, Ofer},
	year = {2010},
	doi = {10.1007/978-3-642-03311-7},
}

@misc{leonard_properties_2013,
	title = {Some properties of path measures},
	doi = {10.48550/arXiv.1308.0217},
	publisher = {arXiv},
	author = {L{\'e}onard, Christian},
	month = aug,
	year = {2013},
	note = {arXiv:1308.0217 [math]},
}

@article{ihara_continuum_2023,
	title = {Continuum {Limits} of {Coupled} {Oscillator} {Networks} {Depending} on {Multiple} {Sparse} {Graphs}},
	volume = {33},
	issn = {1432-1467},
	doi = {10.1007/s00332-023-09921-1},
	number = {4},
	journal = {Journal of Nonlinear Science},
	author = {Ihara, Ryosuke and Yagasaki, Kazuyuki},
	month = jun,
	year = {2023},
	pages = {62},
}

@article{chasseigne_asymptotic_2006,
	title = {Asymptotic behavior for nonlocal diffusion equations},
	volume = {86},
	issn = {0021-7824},
	doi = {10.1016/j.matpur.2006.04.005},
	number = {3},
	journal = {Journal de Math{\'e}matiques Pures et Appliqu{\'e}es},
	author = {Chasseigne, Emmanuel and Chaves, Manuela and Rossi, Julio D.},
	month = sep,
	year = {2006},
	pages = {271--291},
}

@article{nakagaki_maze-solving_2000,
	title = {Maze-solving by an amoeboid organism},
	volume = {407},
	copyright = {https://www.springernature.com/gp/researchers/text-and-data-mining},
	issn = {0028-0836, 1476-4687},
	doi = {10.1038/35035159},
	number = {6803},
	journal = {Nature},
	author = {Nakagaki, Toshiyuki and Yamada, Hiroyasu and T{\'o}th, {\'A}gota},
	month = sep,
	year = {2000},
	pages = {470--470},
}

@article{erbar_gradient_2014,
	title = {Gradient flows of the entropy for jump processes},
	volume = {50},
	issn = {0246-0203},
	doi = {10.1214/12-AIHP537},
	number = {3},
	journal = {Annales de l'Institut Henri Poincar{\'e}, Probabilit{\'e}s et Statistiques},
	author = {Erbar, Matthias},
	month = aug,
	year = {2014},
	pages = {920--945},
}

@article{peletier_jump_2022,
	title = {Jump processes as generalized gradient flows},
	volume = {61},
	issn = {1432-0835},
	doi = {10.1007/s00526-021-02130-2},
	number = {1},
	journal = {Calculus of Variations and Partial Differential Equations},
	author = {Peletier, Mark A. and Rossi, Riccarda and Savar{\'e}, Giuseppe and Tse, Oliver},
	month = jan,
	year = {2022},
	pages = {33},
}

@book{kuramoto_chemical_1984,
	address = {Berlin, Heidelberg},
	series = {Springer {Series} in {Synergetics}},
	title = {Chemical {Oscillations}, {Waves}, and {Turbulence}},
	volume = {19},
	copyright = {http://www.springer.com/tdm},
	isbn = {978-3-642-69691-6 978-3-642-69689-3},
	publisher = {Springer},
	author = {Kuramoto, Yoshiki},
	editor = {Haken, Hermann},
	year = {1984},
	doi = {10.1007/978-3-642-69689-3},
}

@article{lavenant_toward_2024,
	title = {Toward a mathematical theory of trajectory inference},
	volume = {34},
	issn = {1050-5164, 2168-8737},
	doi = {10.1214/23-AAP1969},
	number = {1A},
	journal = {The Annals of Applied Probability},
	author = {Lavenant, Hugo and Zhang, Stephen and Kim, Young-Heon and Schiebinger, Geoffrey},
	month = feb,
	year = {2024},
	pages = {428--500},
}

@misc{conforti_quantitative_2024,
	title = {Quantitative contraction rates for {Sinkhorn} algorithm: beyond bounded costs and compact marginals},
	shorttitle = {Quantitative contraction rates for {Sinkhorn} algorithm},
	doi = {10.48550/arXiv.2304.04451},
	publisher = {arXiv},
	author = {Conforti, Giovanni and Durmus, Alain and Greco, Giacomo},
	month = jun,
	year = {2024},
	note = {arXiv:2304.04451 [math]},
}

@article{schiebinger_optimal-transport_2019,
	title = {Optimal-{Transport} {Analysis} of {Single}-{Cell} {Gene} {Expression} {Identifies} {Developmental} {Trajectories} in {Reprogramming}},
	volume = {176},
	issn = {0092-8674},
	doi = {10.1016/j.cell.2019.01.006},
	number = {4},
	journal = {Cell},
	author = {Schiebinger, Geoffrey and Shu, Jian and Tabaka, Marcin and Cleary, Brian and Subramanian, Vidya and Solomon, Aryeh and Gould, Joshua and Liu, Siyan and Lin, Stacie and Berube, Peter and Lee, Lia and Chen, Jenny and Brumbaugh, Justin and Rigollet, Philippe and Hochedlinger, Konrad and Jaenisch, Rudolf and Regev, Aviv and Lander, Eric S.},
	month = feb,
	year = {2019},
	pages = {928--943.e22},
}

@book{villani_optimal_2009,
	address = {Berlin},
	series = {Grundlehren der mathematischen {Wissenschaften}},
	title = {Optimal transport: old and new},
	isbn = {978-3-540-71049-3},
	shorttitle = {Optimal transport},
	number = {338},
	publisher = {Springer},
	author = {Villani, C{\'e}dric},
	year = {2009},
	doi = {10.1007/978-3-540-71050-9},
}

@inproceedings{follmer_random_1988,
	address = {Berlin, Heidelberg},
	title = {Random fields and diffusion processes},
	isbn = {978-3-540-46042-8},
	doi = {10.1007/BFb0086180},
	booktitle = {{\'E}cole d'{\'E}t{\'e} de {Probabilit{\'e}s} de {Saint}-{Flour} {XV}--{XVII}, 1985--87},
	publisher = {Springer},
	author = {F{\"o}llmer, Hans},
	year = {1988},
	pages = {101--203},
}

@misc{albergo_stochastic_2023,
	title = {Stochastic {Interpolants}: {A} {Unifying} {Framework} for {Flows} and {Diffusions}},
	shorttitle = {Stochastic {Interpolants}},
	doi = {10.48550/arXiv.2303.08797},
	publisher = {arXiv},
	author = {Albergo, Michael S. and Boffi, Nicholas M. and Vanden-Eijnden, Eric},
	month = nov,
	year = {2023},
	note = {arXiv:2303.08797 [cs]},
}

@book{groot_non-equilibrium_2013,
	address = {Newburyport},
	series = {Dover {Books} on {Physics}},
	title = {Non-equilibrium thermodynamics},
	isbn = {978-0-486-64741-8 978-0-486-15350-6},
	language = {eng},
	publisher = {Dover Publications},
	author = {Groot, Sybren Ruurds de and Mazur, Peter},
	year = {2013},
}

@inproceedings{gunasekar_mirrorless_2021,
	address = {Virtual},
	title = {Mirrorless {Mirror} {Descent}: {A} {Natural} {Derivation} of {Mirror} {Descent}},
	issn = {2640-3498},
	shorttitle = {Mirrorless {Mirror} {Descent}},
	url = {https://proceedings.mlr.press/v130/gunasekar21a.html},
	booktitle = {Proceedings of {The} 24th {International} {Conference} on {Artificial} {Intelligence} and {Statistics}},
	publisher = {PMLR},
	author = {Gunasekar, Suriya and Woodworth, Blake and Srebro, Nathan},
	month = mar,
	year = {2021},
	pages = {2305--2313},
}

@article{liero_microscopic_2017,
	title = {On microscopic origins of generalized gradient structures},
	volume = {10},
	issn = {1937-1179},
	doi = {10.3934/dcdss.2017001},
	number = {1},
	journal = {Discrete \& Continuous Dynamical Systems - S},
	author = {Liero, Matthias and Mielke, Alexander and Peletier, Mark A. and Renger, D. R. Michiel},
	year = {2017},
	pages = {1--35},
}

@misc{conforti_projected_2023,
	title = {Projected {Langevin} dynamics and a gradient flow for entropic optimal transport},
	doi = {10.48550/arXiv.2309.08598},
	publisher = {arXiv},
	author = {Conforti, Giovanni and Lacker, Daniel and Pal, Soumik},
	month = sep,
	year = {2023},
	note = {arXiv:2309.08598 [math]},
}

@article{eckstein_hilberts_2025,
	title = {Hilbert’s projective metric for functions of bounded growth and exponential convergence of {Sinkhorn}’s algorithm},
	issn = {1432-2064},
	doi = {10.1007/s00440-025-01366-9},
	journal = {Probability Theory and Related Fields},
	author = {Eckstein, Stephan},
	month = mar,
	year = {2025},
}

@misc{mielke_notes_2025,
	title = {Some notes on the {Hellinger} distance and various {Fisher}-{Rao} distances},
	doi = {10.48550/arXiv.2510.02537},
	publisher = {arXiv},
	author = {Mielke, Alexander},
	month = oct,
	year = {2025},
	note = {arXiv:2510.02537 [math-ph]},
}
\bibliographystyle{unsrtnat}

\end{document}